\documentclass[twoside]{article}
\usepackage[utf8]{inputenc}
\usepackage[T1]{fontenc}

\usepackage{enumitem}
\usepackage{amsmath}
\usepackage{amssymb}
\usepackage{amsthm}
\usepackage{amsfonts}
\usepackage{graphicx}
\usepackage{mathtools}
\usepackage{tcolorbox}
\usepackage{float}
\usepackage{todonotes}
\usepackage{booktabs}
\usepackage{tabularx}
\usepackage{array}
\usepackage{biblatex}
\usepackage{algorithm}
\usepackage{subcaption}
\usepackage{placeins}
\usepackage[noend]{algpseudocode}
\usepackage{hyperref}

\newcolumntype{C}[1]{>{\centering\arraybackslash}p{#1}}

\newcommand{\paren}[1]{\left( #1 \right)}
\newcommand{\sqbrac}[1]{\left[ #1 \right]}
\newcommand{\set}[1]{\left\{ #1 \right\}}
\newcommand{\abs}[1]{\left\vert #1 \right\vert}
\newcommand{\norm}[2]{\left\Vert #1 \right\Vert_{#2}}

\newcommand{\calA}{\mathcal{A}}
\newcommand{\calD}{\mathcal{D}}
\newcommand{\calF}{\mathcal{F}}
\newcommand{\calG}{\mathcal{G}}
\newcommand{\calN}{\mathcal{N}}
\newcommand{\calO}{\mathcal{O}}
\newcommand{\calU}{\mathcal{U}}
\newcommand{\calX}{\mathcal{X}}
\newcommand{\calY}{\mathcal{Y}}

\newcommand{\bbE}{\mathbb{E}}
\newcommand{\bbI}{\mathbb{I}}
\newcommand{\bbN}{\mathbb{N}}
\newcommand{\bbP}{\mathbb{P}}
\newcommand{\bbR}{\mathbb{R}}

\newtheorem{theorem}{Theorem}
\newtheorem{lemma}[theorem]{Lemma}
\newtheorem{proposition}[theorem]{Proposition}
\newtheorem{definition}[theorem]{Definition}

\newcommand{\Var}{\text{Var}}
\newcommand{\Vup}{V^{\uparrow}}
\newcommand{\vbar}{\overline{v}}
\newcommand{\Vbar}{\overline{V}}
\newcommand{\Vdag}{V^{\dagger}}

\newcommand{\bfp}{\mathbf{p}}
\newcommand{\bfw}{\mathbf{w}}
\newcommand{\bfx}{\mathbf{x}}

\newcommand{\bfell}{\boldsymbol{\ell}}

\newcommand{\LineComment}[1]{\State \textcolor{gray}{\textit{// #1}}}

\newif\ifpoints
\pointsfalse
\usepackage[preprint]{aistats2026}
\begin{document}



\twocolumn[

\aistatstitle{OEUVRE: OnlinE Unbiased Variance-Reduced loss Estimation}

\aistatsauthor{Kanad Pardeshi \And Bryan Wilder \And Aarti Singh}

\aistatsaddress{Machine Learning Department \\ Carnegie Mellon University \And Machine Learning Department \\ Carnegie Mellon University \And Machine Learning Department \\ Carnegie Mellon University} ]

\begin{abstract}
% Online learning algorithms update their learned function as they train on data, and an effective estimate of the expected loss at the current time step is a naturally useful quantity. Current methods either do these evaluations at sparse time steps (holdout methods), or use the interleaved test-then-train approach for estimation at all time steps. These estimators suffer from bias-variance tradeoffs, have increasing time complexity for updates, or do not have explicit convergence guarantees. In this work, we present OEUVRE, an estimator which uses careful evaluation of the incoming sample on \textit{two} past models to recursively update the loss estimate. We connect estimation with algorithmic stability, a property satisfied by many commonly used online learning algorithms, and prove consistency and closed-form convergence rates for this estimator in the i.i.d. case. We further prove asymptotic convergence and concentration bounds, with these rates being closely related to the algorithmic stability rates. We test the practical effectiveness of OEUVRE on a variety of online and stochastic tasks, observing that it performs competitively against previously used estimators.

Online learning algorithms continually update their models as data arrive, making it essential to accurately estimate the expected loss at the current time step. The prequential method is an effective estimation approach which can be practically deployed in various ways. However, theoretical guarantees have previously been established under strong conditions on the algorithm, and practical algorithms have hyperparameters which require careful tuning. We introduce OEUVRE, an estimator that evaluates each incoming sample on the function learned at the current and previous time steps, recursively updating the loss estimate in constant time and memory. We use algorithmic stability, a property satisfied by many popular online learners, for optimal updates and prove consistency, convergence rates, and concentration bounds for our estimator. We design a method to adaptively tune OEUVRE's hyperparameters and test it across diverse online and stochastic tasks. We observe that OEUVRE matches or outperforms other estimators even when their hyperparameters are tuned with oracle access to ground truth.

% Existing approaches, such as holdout evaluation or the interleaved test-then-train method, either provide loss estimates only at sparse intervals, suffer from fixed bias-variance tradeoffs, or lack explicit convergence guarantees. 

% Our key insight is to connect estimation with algorithmic stability, a property satisfied by many popular online learners. 
% We leverage algorithmic stability, a property satisfied by many popular online learners, to `
% This connection enables us to prove consistency, closed-form convergence rates, and sharp concentration bounds for the loss estimate assuming the learning algorithm is trained on i.i.d. data. The rates of convergence of these results are closely linked to stability parameters. 
% Our key insight is to connect estimation with algorithmic stability, a property satisfied by many popular online learners. 
% Empirically, we demonstrate that OEUVRE provides accurate estimates across diverse online and stochastic tasks, outperforming or matching previously used estimators even when their hyperparameters are tuned with access to ground truth. % in hindsight. 

% in practice.

\end{abstract}

\section{Introduction}

\begin{figure*}[th]
    \centering
    \includegraphics[width=0.9\textwidth]{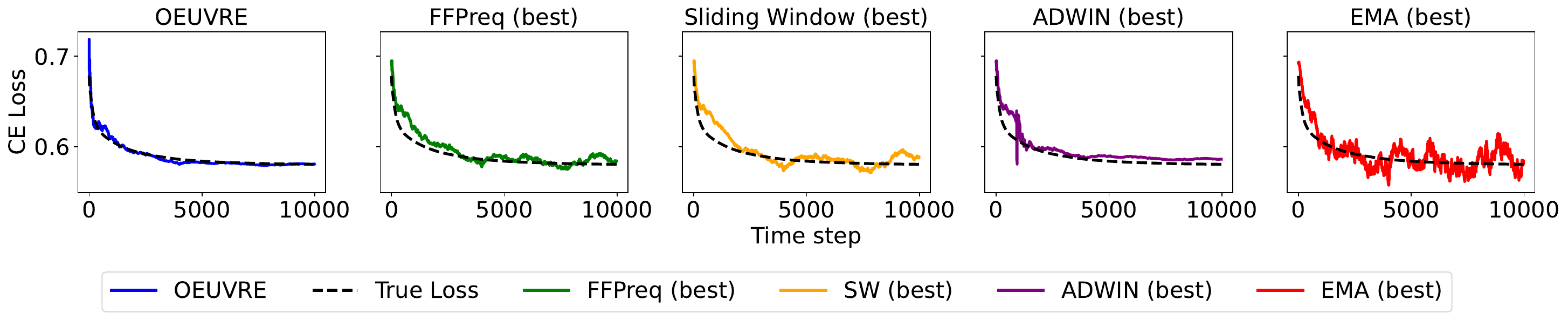}
    \caption{We illustrate the behavior of OEUVRE and several baselines on a representative run of the Diabetes Health Indicators dataset. The hyperparameter for each baseline was chosen using grid search to minimize RMSE. We see that our proposed estimator provides a more accurate continuous estimate of the true loss compared to the baselines without the need for hyperparameter tuning.}
    \label{fig:intro_fig}
    \vspace{-1em}
\end{figure*}

Online learning is an important part of modern-day machine learning, and algorithms that update their models over time are widely used in applications such as recommendation algorithms and portfolio selection in online time series. In such settings, a measure of the model's \textit{expected loss} at the current time step is a naturally useful quantity for tasks like monitoring model performance over time. This measure can also be used to determine early stopping in training when the labels are costly to obtain, and when we would want to stop training due to required compute cost or time once the model is sufficiently accurate. Moreover, a good estimate of the model's current performance can play a crucial role in online model selection \cite{foster2017parameter}.

A common approach to online loss estimation is the prequential method \cite{dawid1984present, vovk2005algorithmic}, which enables performance estimation for online models in a sample-efficient way. Here, the incoming sample is first used to evaluate the current model's performance before using it to update the model. The prequential method can be used with different methods of combining evaluations, such as fading factors, sliding windows, \cite{gama2013evaluating} and exponential moving averages. Moreover, drift detection methods like ADWIN \cite{bifet2007learning}, EDDM \cite{baena2006early}, and Page-Hinkley \cite{page1954continuous, sebastiao2017supporting} can be modified for online loss estimation, resulting in adaptive sliding windows.

A key challenge in online loss estimation is that while the incoming data may be i.i.d., the observed sequence of losses on these samples is not i.i.d. as the learned functions evolve over time. Existing theoretical results for prequential estimation \cite{gama2013evaluating} require the strong assumption that the sequence of learned functions $f_t$ converges to a fixed limit $f^*$, and show that the estimate is consistent with this assumption. Drift detection methods \cite{bifet2007learning, dos2016fast} provide detection guarantees for general data streams but do not exploit properties of the learning algorithm for convergence guarantees. In practice, these methods have hyperparameters which need to be set carefully for accurate loss estimation.

In this work, we present OnlinE Unbiased Variance-Reduced loss Estimation (OEUVRE), an online loss estimator which evaluates the incoming sample on the current model \textit{and} the model from the previous time step. We use these two evaluations to update the estimate, with the past evaluation helping to reduce the variance of the updated estimate. The influence of the past estimate is controlled by a sequence of weights $(\gamma_{t})$, where each $\gamma_{t}$ is chosen to minimize a variance bound. This minimization requires knowledge of the extent of change in the learned function between two consecutive time steps, and \textit{algorithmic stability} provides us with exactly that information. Stability bounds are known for many popular online algorithms such as Follow-the-Leader (FTL), Follow-the-Regularized-Leader (FTRL), and Online Mirror Descent (OMD), and OEUVRE can readily be used with these frameworks.

The simple step of evaluating the incoming sample on an additional model leads to powerful theoretical results when coupled with algorithmic stability. We show that in the setting with i.i.d. data and for losses with bounded variance, OEUVRE converges in $L^2$ to the expected loss of the currently learned model, along with proving convergence rates that scale with the stability rate of the algorithm. We prove that OEUVRE forms a martingale, which enables us to establish novel asymptotic convergence and concentration results for online loss estimation. The rates in our results are closely related to the rate of algorithmic stability, thus linking statistical performance to optimization dynamics. To our knowledge, this is the first work to systematically leverage the stability of learning algorithms for online performance estimation.

We then design a method to improve OEUVRE's practical performance without affecting its rate of convergence by adaptively estimating its hyperparameters from observed data. We demonstrate the empirical effectiveness of our estimator through experiments on linear regression, logistic regression, prediction with expert advice, and neural networks. We observe that OEUVRE consistently provides accurate estimates of the expected loss of the currently learned model across tasks and noise regimes. OEUVRE matches or outperforms existing baselines, even when the hyperparameters of each baseline are tuned for optimal estimation. We provide a comparison of OEUVRE against baseline methods for a representative run in Figure \ref{fig:intro_fig}.

\vspace{-0.5em}
\section{Related Work}
\vspace{-0.5em}

To monitor online performance at every step, practitioners commonly use prequential (interleaved test-then-train) evaluation, yielding estimators such as the plain prequential average, sliding windows, fading-factor variants, and exponential moving averages (EMA) \cite{gama2013evaluating}. The prequential average is consistent under i.i.d. data when the learned predictors $f_t$ converge to a limit $f^\star$ \cite{gama2013evaluating}. Moreover, although concentration bound can be obtained for the prequential average \cite{gama2009issues}, these bound hold for the \textit{time-averaged} loss, and not the loss for the current time step.  Our approach instead leverages algorithmic stability, a much weaker convergence assumption, to design an estimator for the current model's expected loss with stronger $L^2$ convergence given i.i.d. data. Sliding windows and fixed-decay EMAs impose a fixed bias–variance trade-off; OEUVRE can adapt to the problem setting by reducing the variance proportional to the difficulty of the problem through the stability of the learning algorithm and provide a consistent estimate.

% Under i.i.d. data and standard stability assumptions, OEUVRE yields a consistent, low-variance estimate with explicit rates, along with updates in constant time and memory. 

% \textbf{Drift detection.} 
Drift detection methods like ADWIN \cite{bifet2007learning}, EDDM \cite{baena2006early}, Page-Hinkley \cite{page1954continuous, sebastiao2017supporting} and KSWIN \cite{dos2016fast} are designed to identify changes in the data distribution through hypothesis testing or thresholding on observed values. While these methods can be modified to create adaptive sliding window estimators for loss estimation, they are primarily designed for non-stationary environments with concept drift. In contrast, OEUVRE is specifically designed for loss estimation, providing explicit convergence guarantees that exploit algorithmic stability properties of the learner. Moreover, the sensitivity of drift detection methods depends on tunable hyperparameters, whereas OEUVRE's hyperparameters can be set adaptively (Section \ref{sec:adaptive_const_est}) without affecting convergence rates given that the model is trained on i.i.d. data.

% \textbf{Algorithmic stability.} 
Algorithmic stability \cite{bousquet2002stability, bousquet2020sharper} bounds the extent to which a new sample can change the function learned by a learning algorithm on a dataset. Uniform stability has previously been connected to generalization bounds \cite{bousquet2002stability, bousquet2020sharper}. Stability guarantees are known for a wide variety of commonly used algorithms, including regularized Empirical Risk Minimization \cite{bousquet2002stability}, Stochastic Gradient Descent \cite{hardt2016train}, and online learning frameworks \cite{sahaInterplayStabilityRegret2012}. Our work is the first to connect uniform stability with loss estimation, and our theoretical guarantees establish a link between the estimator's convergence rate and the algorithm's stability rate. We list some useful known stability bounds in Table \ref{tab:stability_results} and elaborate on this connection in Section \ref{sec:known_stability_results}.

% \textbf{Variance reduction methods.} 
Variance reduction techniques \cite{botev2017variance, gower2020variance} and stochastic approximation \cite{borkar2008stochastic, lai2003stochastic} have been highly successful in stochastic optimization. They underpin methods for stochastic gradient descent \cite{johnson2013accelerating, defazio2014saga, shalev2013stochastic, schmidt2017minimizing} and Q-learning \cite{wainwright2019variance, khamaru2021instance, wang2024sample}, where they yield faster convergence rates. Broad mechanisms include storing snapshots and gradient tables, as well as classical control-variate and importance-sampling strategies. In contrast, our work adapts the spirit of variance reduction to the setting of \textit{loss estimation}: we leverage evaluations at consecutive time steps to more accurately estimate the expected performance of the learning algorithm at the current step.

% \begin{enumerate}
%   \item We show that OEUVRE is unbiased and consistent by deriving a martingale decomposition of the estimator.
%   \item We derive a closed-form upper bound of the variance, providing us with a rate of convergence of the estimator. The rate of convergence is intimately related to the rate of algorithmic stability, thus linking statistical performance to optimization dynamics.
%   \item We further establish asymptotic normality by proving a Central Limit Theorem (CLT) result and derive fixed-time and time-uniform concentration bounds. The rates of these results are also dictated by algorithmic stability.
%   \item Finally, we provide a method to improve the estimator's empirical performance by adaptively estimating the constant factors for algorithmic stability without affecting the estimator's rate of convergence.
% \end{enumerate}

 % We prove asymptotic and finite-time guarantees on the variance of the estimator, addressing another gap in existing literature. 

\section{Problem Setup}
We assume a sequence of samples $\set{z_t}_{t \in \bbN^+}$ arriving in an online fashion, where $z_t = (x_t, y_t)$, $x_t \in \calX$ are the features, $y_t \in \calY$ are the labels, and $(x_t, y_t)$ are drawn i.i.d. from the distribution $\calD$ over $\calX \times \calY$. The i.i.d. assumption is standard in analyzing online loss estimation \cite{gama2013evaluating}. Using these samples, we learn a sequence of functions $\paren{f_t}_{t \geq 1}$ using some online learning algorithm, where function $f_t: \calX \to \calY$ is learned using the first $(t-1)$ samples, $\set{z_s}_{s \in [t-1]}$. We further assume that all $f_t$'s belong to a hypothesis class $\calG$. We define $\set{\calF_t}_{t \in \bbN}$ as the canonical filtration defined on the sample sequence $\set{z_t}$.

We measure the performance of the learning algorithm using a loss function $\ell: \calX \times \calY \to \bbR^+$. We define $\ell_t(z) = \ell_t((x, y)) = \ell(f_t(x), y)$. The expected loss of $f_t$ is given by $\bbE_{z \sim \calD} \sqbrac{\ell_t(z)} = \bbE\sqbrac{\ell(f_t(x), y) \mid \calF_{t-1}}$. Our expected loss is conditioned on the first $(t-1)$ samples, therefore $f_t$ is constant inside the expectation.

Our goal is to estimate $\bbE_{z \sim \calD}[\ell_{t}(z)]$. Although samples $\set{z_{t}}$ are drawn i.i.d., the sequence of loss functions $\ell_{t}(\cdot)$ changes as a result of learning $f_{t}$ from the data. Thus any sample $\ell_{t}(z_{t})$ that we draw will not be i.i.d. with $t$. We assume that $\Var(\ell_t \mid \calF_{t-1})$ exists and is upper bounded by $b^2$ for some $b > 0$. This mild assumption  is satisfied by a wide variety of loss functions, including bounded loss functions and bounded hypothesis classes. Throughout the paper, we distinguish between \textit{conditional} and \textit{total} expectation. The conditional expectation $\bbE[\cdot \mid \calF_{t-1}]$ assumes that the first $(t-1)$ samples are fixed and the expectation is restricted to the $t$-th sample. On the other hand, the total expectation $\bbE[\cdot]$ is over all $t$ samples. Similar definitions hold for conditional and total variance.

\subsection{The OEUVRE Estimator}
The OEUVRE estimator provides a sequence of estimates $\set{L_t}$, where $L_t$ is an estimate of $\bbE_{z \sim \calD}[\ell_{t}(z)]$. The estimator is recursively defined as follows:
\begin{equation}
  \label{eq:oeuvre_defn}
  \begin{aligned}
    L_1 &\coloneqq \ell_1(z_1) \\
    L_t &\coloneqq \ell_t(z_t) + (1 - \gamma_t) \sqbrac{L_{t-1} - \ell_{t-1}(z_t)}
  \end{aligned}
\end{equation}
Here, $\ell_t(z_t)$ is the evaluation of the performance of $f_t$ on the future sample $(z_t)$, and $\ell_{t-1}(z_t)$ is evaluating the performance of $f_{t-1}$ on $z_t$. $\paren{\gamma_t}_{t \geq 2}$ is a sequence of \textit{pre-determined} weights such that each $\gamma_t$ lies in $[0, 1]$. We observe that each function $f_t$ is evaluated on two samples, $z_t$ and $z_{t+1}$. Both these samples lie outside of the training data of $f_t$, since it is trained on $\set{z_1, \ldots, z_{t-1}}$. Conversely, each sample $z_t$ is used for evaluation for two functions, $f_{t-1}$ and $f_t$. This is in stark contrast with previous evaluation methods, which only evaluate each sample on exactly one function.

% Since both $\ell_t$ and $\ell_{t-1}$ are computed on the same future sample $(x_t, y_t)$, the two quantities are highly correlated. We can interpret Equation \ref{eq:oeuvre_defn} from the perspective of variance reduction, where the term multiplied by $(1 - \gamma_t)$ is used to reduce the variance of the first term.

Algorithm \ref{alg:oeuvre} summarizes the OEUVRE estimator.  Each recursive update of the OEUVRE estimator requires $\calO(1)$ time and memory, making it highly efficient. We emphasize the order of evaluation and function updates: we first calculate $\ell_{t-1}(z_t)$ using $f_{t-1}$; then update $f_{t-1}$ to $f_t$ using $z_{t-1}$; finally, calculate $\ell_t(z_t)$ using $f_t$. Thus, at time step $t$, the incoming sample $z_t$ is used only for evaluation. $z_t$ is used for training at the next time step $t + 1$. This ordering is a feature of prequential evaluation, where the incoming sample is used for testing and estimating the loss of the learned function at that time step.

% The two evaluations on the same model allow us to develop a martingale from the estimator, which is crucial for our theoretical analysis.

% Computationally, the estimator can be paired with any (stable) online learning algorithm with the following update rule for time $t$:
% \fbox{
%   \begin{minipage}{0.9\linewidth}
%   \textbf{Update rule for OEUVRE}:
%     \begin{enumerate}
%       \item Compute $\ell_{t-1}(z_t)$ using $f_{t-1}$.
%       \item Use the sample $z_{t-1}$ to update $f_{t-1}$ to $f_t$.
%       \item Compute $\ell_t(z_t)$ using $f_t$, and use Equation \ref{eq:oeuvre_defn} to obtain an estimate of the expected loss of $f_t$.
%     \end{enumerate}
%   \end{minipage}
%   }

\begin{algorithm}
  \caption{OEUVRE estimator}
  \label{alg:oeuvre}
  \begin{algorithmic}[1]
    \State \textbf{Input}: Data points $z_t = (x_t, y_t)$, loss function $\ell$, learning algorithm $A: \calG \times (\calX \times \calY) \to \calG$ % that returns $f_t = A(f_{t-1}, z_t)$.
    \State \textbf{Parameters}: $c_0 = 1, b_0 = 2$, burn-in period $t_0$, small constant $\epsilon$ % Stability constant $\hat{c}$, variance constant $\hat{b}$, burn-in period $t_0$
    \State $\hat{c} \gets c_0$, $\hat{b} \gets b_0$, $C = \emptyset$, $B = \emptyset$
    \For {$t \in \bbN$}
    \LineComment{Initialize/reset after burn-in}
    \If {$t = 1$ or $t = t_0$} 
    \State $L_t \gets \ell_t(z_t) = \ell(f_t(z_t))$
    \State \textbf{continue}
    \EndIf
    % \State Set weights $\gamma_t^*$ using Equation \ref{eq:gamma_optimal} with $\sigma_t = \hat{c} r(t)$, $b = \hat{b}$.

    \LineComment{Evaluation and function update}
    \State $\ell_{t-1}(z_t) \gets \ell(f_{t-1}(x_t, y_t))$ 
    % \State Use $\calA$ to update $f_{t-1}$ to $f_t$ using $z_{t-1}$
    \State $f_t \gets A(f_{t-1}, z_{t-1})$
    \State $\ell_t(z_t) \gets \ell(f_t(x_t, y_t))$

    \LineComment{OEUVRE estimate update}
    \State Set $\gamma_t^*$ with $\sigma_t = \hat{c} r(t)$, $b = \hat{b}$ in Equation \ref{eq:gamma_optimal}
    \State $L_t = \ell_t(z_t) + (1 - \gamma_t^*)(L_{t-1} - \ell_{t-1}(z_t))$
    
    \LineComment{Estimated constants update (Sec. \ref{sec:adaptive_const_est})}
    \If {$t < t_0$} 
    \State $C \gets C \cup (\ell_t(z_t) - \ell_{t-1}(z_t)) / r(t)$ 
    \State $B \gets B \cup \ell_t(z_t)$ 
    \State $\hat{c}^2 = \max(\epsilon, \widehat{\Var}(C))$, $\hat{b}^2 = \max(\epsilon, \widehat{\Var}(B))$
    \EndIf
    \EndFor
  \end{algorithmic}
\end{algorithm}

We define $M_t = L_t - \bbE[\ell_t(Z) \mid \calF_{t-1}]$, which is our estimator $L_t$ centered around the conditional expectation of the loss. This centered sequence is crucial to our analysis, $L_t$ accurately estimating the conditional mean $\bbE[\ell_t(Z) \mid \calF_{t-1}]$ means that $M_t$ is close to zero.

\begin{proposition}{(Unbiasedness in total expectation)}
  \label{prop:unbiasedness}
  For all $t \geq 1$, $\bbE[M_t] = 0$. Thus, $\bbE[L_t] = \bbE[\ell_t(Z)]$.
\end{proposition}
This result is a consequence of a martingale constructed using $M_t$ (Proposition \ref{prop:mtgle_prop} in Section \ref{sec:convergence_and_concentration}). Thus, our estimator is unbiased in \textit{total expectation}. This is true irrespective of the choice of $(\gamma_t)$. While the estimator is not necessarily conditionally unbiased, we later show that with the right choice of $(\gamma_t)$, $L_t$ is a consistent estimator for $\bbE[\ell_t(Z) \mid \calF_{t-1}]$.

\textbf{Choice of $(\gamma_t)$.} Intuitively, $\gamma_t$ controls how much of the information from previous evaluations is carried over to the next time step. A lower $\gamma_t$ implies more weight being assigned to the past estimate, which is desirable if $f_t$ is not too different from $f_{t-1}$. On the other hand, $\gamma_t$ being close to one reduces the impact of the past estimate on the current one. This is desirable if $f_t$ is very different from $f_{t-1}$, and a higher $\gamma_t$ translates to less weight being given to the `memory' consisting of past evaluations. 

% Equation \ref{eq:oeuvre_defn} can also be understood from a variance reduction perspective, since $\ell_{t-1}(z_t)$ would be highly correlated with $\ell_t(z_t)$ as the extent of function updates becomes smaller with time.

A natural choice is setting the weight sequence $(\gamma_t)$ such that the variance of the estimator is minimized. The following proposition establishes a recursion-based upper bound on the total variance of $L_t$.
\begin{proposition}
  \label{prop:var_recursion}
  Let 
  \begin{align*}
    \sigma_t^2 \geq \sup_{(z_1, \ldots, z_{t-1})} \Var(\ell_t(X, Y) - \ell_{t-1}(X, Y) \mid \calF_{t-1})
  \end{align*}
  be a deterministic sequence. Define the following recursion:
  \begin{equation}
    \label{eq:var_recursion}
    \begin{aligned}
      \Vup_1 &\coloneqq b^2 \\ % \Var(\ell_1(Z) \mid \calF_0) \\
      \Vup_t &\coloneqq (\gamma_t b + (1 - \gamma_t) \sigma_t)^2 + (1 - \gamma_t)^2 \Vup_{t-1}
    \end{aligned}
  \end{equation}
  Then, for all $t \geq 1$, 1) $\Var(M_t) \leq \Vup_t$, and 2) $\sum_{i=1}^{t} \Var(M_i \mid \calF_{i-1}) \leq \Vup_t$.
\end{proposition}
We prove this result in Appendix \ref{apdx:var_recursion_proof}. $\Vup_{t}$ acts as an upper bound on both the total variance and the sum of conditional variances of $M_{t}$. We note that this recursion is in terms of $\sigma_t^2 \geq \sup \Var(\ell_t(Z) - \ell_{t-1}(Z) \mid \calF_{t-1})$. This is an upper bound on the conditional variance of the difference between the losses of two consecutive functions evaluated on the same point, $\ell_t(Z)$ and $\ell_{t-1}(Z)$. The key challenge in finding the right $\gamma_t$ is knowing $\sigma_t^2$. As we shall see in Section \ref{sec:known_stability_results}, $\sigma_t^2$ can be upper-bounded using results from algorithmic stability. For a large variety of online learning algorithms, we can show that there is a known sequence $\paren{\sigma_t}$ which satisfies the condition in Proposition \ref{prop:var_recursion} for all $t$. Since we require $\paren{\gamma_t}$ to be a deterministic sequence, we make use of these upper bounds to obtain our optimal $\paren{\gamma_t}$.

% This quantity is intimately related to the algorithmic stability of the learning algorithm, and we elaborate on this relation in Section \ref{sec:known_stability_results}. 
% , which is the variance of the difference between errors made by functions at two consecutive time steps. This is the total variance, i.e., the variance over all $t$ time steps. 

Equation \ref{eq:var_recursion} is quadratic in $\gamma_t$, and solving for $\gamma_t$ minimizing $\Vup_t$ at each step gives us the following optimal choice of $\gamma_t$:
\begin{equation}
  \label{eq:gamma_optimal}
  \begin{aligned}
    \gamma_t^* = \begin{cases}
      0, &\quad \Vup_{t-1} \leq \sigma_t (b - \sigma_t) \\
      1, &\quad \sigma_t \geq b \\
      \frac{\Vup_{t-1} - \sigma_t(b - \sigma_t)}{\Vup_{t-1} + (b - \sigma_t)^2}, &\quad \text{otherwise}
    \end{cases}
  \end{aligned}
\end{equation}
% Thus, if $\sigma_t$ is too far away from $0$ or $b$, we set $\gamma_t$ to $0$, which means that information from the past is carried over as is. On the other hand, if $\sigma_t$ is too large, $\gamma_t = 1$ means that the past information is discarded, and the estimator is reset. 
% Thus, if 
Intuitively, the first case corresponds to a situation where $\sigma_t$ is moderately large when compared to $V_t$. We can then set $\gamma_t = 1$, which means the past estimate is not down-weighted. On the other hand, $\sigma_t$ is very large in the second case, corresponding to a sudden change in the learned function. We then set $\gamma_t = 1$, effectively resetting the estimator. The third case corresponds to a situation where $\sigma_t$ is of an appropriate size, and we set $\gamma_t$ to a value in $(0,1)$ to minimize the variance.
Using the optimal $\gamma_t$ in Equation \ref{eq:var_recursion}, the optimal variance bound is:
\begin{equation}
  \label{eq:var_optimal}
  \begin{aligned}
    \Vup_t = \begin{cases}
      \sigma_t^2 + V_{t-1}, &\quad \Vup_{t-1} \leq \sigma_t (b - \sigma_t) \\
      b^2, &\quad \sigma_t \geq b \\
      \frac{b^2 \Vup_{t-1}}{\Vup_{t-1} + (b - \sigma_t)^2}, &\quad \text{otherwise}
    \end{cases}
  \end{aligned}
\end{equation}
\vspace{-1em}

We now show that the estimator is consistent with the optimal choice of $\gamma_t$.
\begin{lemma}{(Consistency)}
  \label{lem:opt_gamma_consistency}
  If $\sigma_t \to 0$, then with the optimal choice of $\set{\gamma_t^*}$, $\gamma_t^* \to 0$. Moreover, $\Var(M_t) \to 0$, i.e., $L_t - \bbE\paren{\ell_t(Z) \mid \calF_{t-1}} \stackrel{L^2}{\to} 0$.
\end{lemma}

% Thus, as $\Var(M_t) \to 0$, we have $L_t \to \bbE[\ell_t(x_t, y_t) \mid \calF_{t-1}]$ in the $L_2$ sense. 
We prove this result in Appendix \ref{apdx:opt_gamma_consistency_proof}. Lemma \ref{lem:opt_gamma_consistency} tells us that OEUVRE converges to the \textit{conditional expectation} of the loss, i.e., the expected loss of the currently learned function, in $L^2$. We thus have a stronger consistency result while using a weaker assumption than previous work on prequential estimation \cite{gama2013evaluating}, where it is assumed that $\ell_t \to \ell^*$ to show that the prequential estimator converges to $\bbE[\ell^*(Z)]$.

% This means that not only is our estimator unbiased in total expectation, but it also converges in $L_2$ to the expected loss of the learnt function $f_t$. 

We can also obtain closed-form upper bounds for convergence rates:
\begin{theorem}
  \label{thm:gamma_convergence_rates}
  Let
  \begin{align}
    \label{eq:gamma_constraint}
    \gamma_t = \begin{cases} 1/t &\quad \text{ if } \sigma_t = o(1/t) \text{ or } \sigma_t = 1/t, \\
    \Omega(\sigma_t) &\quad \text{otherwise} \end{cases}
  \end{align}
  Then, with this choice of $\gamma_t$ in Equation \ref{eq:var_recursion}, we have $\Var(M_t) \leq \Vup_t = \calO(\gamma_t) = \calO(\max\set{1/t, \sigma_t})$.
\end{theorem}
We prove this theorem in Appendix \ref{apdx:gamma_convergence_rates_proof}. The $\gamma_t$ designed above can potentially be a non-optimal choice of weights. Since $\gamma_t^*$ is chosen to minimize the variance, $\Var(M_t)$ with the optimal choice of $\gamma_t^*$ would also be upper bounded by this rate. This result establishes the connection between the stability of the learning algorithm and the closed-form convergence rate of the OEUVRE estimator.

To build intuition, we consider the simple example of a sequence of static functions $f_t = f$, where $\sigma_t = \Var(\ell_t(X, Y) - \ell_{t-1}(X, Y)) = 0$ for all $t$. Using this sequence of $\sigma_t$ in Equation \ref{eq:gamma_optimal} gives us $\gamma_t = 1/t$. As $f_t(X_t) = f_{t-1}(X_t) = f(X_t)$, the OEUVRE estimator becomes the empirical mean of the past evaluations. Intuitively, this corresponds to `best-case' behavior where there is no change between subsequent functions and we can expect the decay rates for $\gamma_t^*$ to be slower than $1/t$ when $f_t$'s evolve with time.

% Moreover, using this sequence of $\sigma_t$ in Equation \ref{eq:var_optimal} gives us $V_t = b^2/t$, a natural upper bound on the variance of the empirical mean.
%
% This special case of $f_t = f$ is some form of `best-case' behavior where there is no change between subsequent functions. Intuitively, when all $f_t$'s are not identical, we would expect the weight sequence $\set{\gamma_t}$ to decay slower than $1/t$, since the past estimates carry less information about the current estimate, which would require a larger value of $\gamma_t$ to minimize the variance of $L_t$.

\section{Learning Algorithms with Known Stability Results}
\label{sec:known_stability_results}
\begin{table*}[h]
  \begin{tabular}{C{0.4\textwidth}C{0.3\textwidth}C{0.2\textwidth}}
    \toprule
    \textbf{Algorithmic Paradigm} & \textbf{Example Algorithms} & \textbf{Order of Stability Bound} \\
    \midrule
    Follow the Leader (FTL) & Online Convex Programming (strongly convex functions) & $\calO(1 / t)$ \\
    Follow the Regularized Leader (FTRL) & Online ridge regression & $\calO(1 / \sqrt{t})$ \\
    Regularized Dual Averaging (RDA) & -- & $\calO(1 / \sqrt{t})$ \\
    Implicit Online Learning (IOL) & Implicit Mirror Descent & $\calO(\eta_{t})$ \\
    Online Mirror Descent (OMD) & Online gradient descent, Hedge  & $\calO(\eta_{t})$ \\
    OMD (with Polyak Averaging) & Online gradient descent & $\calO(1 / t)$ \\
    \bottomrule
  \end{tabular}
  \caption{Stability results for popular online learning paradigms \cite{sahaInterplayStabilityRegret2012}, where $\eta_{t}$ corresponds to the learning rate schedule chosen for the algorithm and is typically set to be $\Theta(1/ \sqrt{t})$. We recommend choosing $\sigma_t$ to have the same decay order as the order of the stability bound, with the constant factors estimated adaptively (Section \ref{sec:adaptive_const_est}).}
  \label{tab:stability_results}
  \vspace{-1em}
\end{table*}

In Proposition \ref{prop:var_recursion}, we need information about $\sigma_t^{2} > \sup \Var(\ell_t(Z) - \ell_{t-1}(Z)) \mid \calF_{t-1}$ to determine $\gamma_t$. These bounds can be readily obtained from results on \textit{uniform stability} \cite{bousquet2002stability, bousquet2020sharper, sahaInterplayStabilityRegret2012}.

% \begin{definition}
%   \textit{(Uniform stability)}: Let $\calA$ be a learning algorithm and $S$ be a dataset of size $m$. Let $\calA_S$ be the function learned by $\calA$ on $S$. Let $S^{\setminus m}$ be the dataset $S$ with the $m$-th sample removed. We say that $\calA$ is $\beta_m$-uniformly stable w.r.t. the loss function $\ell$ if, for all $m$ and all datasets $S$ of size $m$,
%   \begin{align*}
%     \norm{\ell(\calA_S(x), y) - \ell(\calA_{S^{\setminus m}}(x), y)}{\infty} &\leq \beta_m,
%   \end{align*}
% \end{definition}

\begin{definition}
  \textit{(Uniform stability)}: Let $\calA$ be a learning algorithm, and $(f_{t})$ be the sequence of functions learned by $\calA$, with $f_{t}$ being learned from samples $\set{z_{s}}_{s = 1}^{t - 1}$. Then, $\calA$ is $\beta_{t}$-uniformly stable if, for all $t$ and for all sequence of samples $\set{z_{s}}_{s=1}^{{t}}$,
  \begin{align*}
    \norm{\ell(f_{t}(x), y) - \ell(f_{t-1}(x), y)}{\infty} &\leq \beta_{t}
  \end{align*}
\end{definition}
\vspace{-1em}

Informally, uniform stability bounds the degree of change that any incoming sample can have on the loss of the learned function at any point $(x, y)$. Since this is a uniform bound, we have $\beta_t^2 >  \Var(\ell_t(Z) - \ell_{t-1}(Z))$ for all $t$, which means that we can use $\sigma_{t}^{2} = \beta_{t}^{2}$ in our variance recursion (Equation \ref{eq:var_recursion}), calculate the optimal weight sequence $\gamma_{t}^{*}$ (Equation \ref{eq:gamma_optimal}), and use these weights in our estimator (Equation \ref{eq:oeuvre_defn}).

Our theory (Lemmas \ref{lem:opt_gamma_consistency}, Theorem \ref{thm:oeuvre_clt}) only requires bounds on $\sup \Var(\ell_t(Z) - \ell_{t-1}(Z) | \calF_{t-1})$, which is a bound on the $L^2$ norm. This is a weaker condition than uniform stability, since the $L^{\infty}$ norm is at least as large as the $L^2$ norm. However, we utilize uniform stability bounds because they are well-established for common algorithms (Table \ref{tab:stability_results}). When available, $L^2$ stability bounds can be directly plugged into our framework, yielding tighter variance recursions and faster convergence rates. The development of systematic $L^2$ stability theory for online learning would immediately translate to improved OEUVRE guarantees.

Uniform stability bounds are known for a wide variety of popular online learning algorithms. We state some important examples in Table \ref{tab:stability_results}.  While we present the order of the stability results, exact stability rates have multiplicative constants based on the Lipschitz constant of the loss function and the hypothesis class, the strong convexity of the loss function, bounds on the function values, etc. Most of these results were established in \cite{sahaInterplayStabilityRegret2012}. As we elaborate in Section \ref{sec:adaptive_const_est}, the rate constants need not be known in practice and can be estimated to achieve good empirical performance.

% We note that uniform stability is a stronger condition than what is required for our estimator to work. We need bounds on $\sup \Var(\ell_t(Z) - \ell_{t-1}(Z) \mid\ \calF_{t-1})$, which is the $L_2$ norm, while uniform stability bounds the $L_{\infty}$ norm. Thus, $L_2$ bounds can improve the performance of the estimator. Since uniform stability results are more common in literature, we use them in our examples of the estimator, and we leave the exploration of $L_2$ stability bounds to future work.
%
%
%
%

\vspace{-0.5em}
\section{Convergence of $(L_{t})$}
\vspace{-0.5em}
\label{sec:convergence_and_concentration}
\textbf{Martingale structure.} We recall that $M_t = L_t - \bbE[\ell_t(z) \mid \calF_{t-1}]$ is the centered version of our estimator $L_t$. We also define $\Gamma_t = \prod_{i = 2}^{t} (1 - \gamma_i)$, along with $\Gamma_1 = 1$. This product naturally arises from expanding the OEUVRE recursion (Equation \ref{eq:oeuvre_defn}), representing the accumulation of the decay factors from all steps of the recursion. Since we assume that $\paren{\gamma_t}$ is a pre-determined sequence, $(\Gamma_t)$ is also deterministic.

\begin{proposition}
  \label{prop:mtgle_prop}
  $\paren{M_t / \Gamma_t}_{t \geq 1}$ is a martingale w.r.t the canonical filtration $\paren{\calF_t}_{t \geq 0}$, with $\bbE[M_1 / \Gamma_1 \mid \calF_0] = 0$
\end{proposition}

We prove this result in Appendix \ref{apdx:mtgle_prop_proof}. This martingale structure allows us to use tools from martingale theory to prove strong convergence results.

\vspace{-0.5em}
\subsection{Asymptotic Convergence}
\vspace{-0.5em}
\label{sec:asymptotic_convergence}
\begin{theorem}
\label{thm:oeuvre_clt}
% Let the following conditions hold:
% Let $\set{U_t}_{t \geq 1}$ be a the martingale difference sequence associated with $M_t / \Gamma_t$, i.e., $U_1 = M_1 / \Gamma_1$, and $U_t = M_t / \Gamma_t - M_{t-1} / \Gamma_{t-1}$ for $t > 1$. Let the following conditions hold:
Let $v^*$ be a random variable taking values in $[\nu^2, b^2]$, where $\nu > 0$. Let the following conditions hold:
\begin{enumerate}[label=(\roman*), leftmargin=*]
    \item $\Var(\ell_t \mid \calF_{t-1}) \in [\nu^2, b^2]$.
    \item $\Var(\ell_t \mid \calF_{t-1}) \stackrel{\text{p}}{\to} v^*$.
    % \item The distribution of $\paren{\ell_t(z) - (1 - \gamma_t)\ell_{t-1}(z)}$ has bounded kurtosis.
    \item $\paren{\ell_t(z) - (1 - \gamma_t)\ell_{t-1}(z)}$ has bounded kurtosis.
    % \item The weight sequence $\set{\gamma_t}$ is such that $\gamma_t = \omega(\sigma_t)$ and $1/\gamma_t - 1/\gamma_{t-1} \leq 1$.
    \item The weight sequence $(\gamma_t)$ satisfies Equation \ref{eq:gamma_constraint} with the modified condition that $\gamma_t = \omega(\sigma_t)$.
\end{enumerate}
  Then, $M_t / \sqrt{V_t} \stackrel{\text{d}}{\to} \calN(0, 1)$, where $V_t = \sum_{i=1}^t \Var(M_i \mid \calF_{i-1})$.
\end{theorem}
We prove this theorem in Appendix \ref{apdx:oeuvre_clt_proof}. Conditions $(i)$ and $(ii)$ ensure that the conditional variances of the losses converge to some non-degenerate random variable. Condition $(iii)$ is a mild restriction on the kurtosis of consecutive difference, and should be satisfied by loss distributions with light tails, such as sub-Gaussian and sub-exponential distributions. Condition $(iv)$ is a restriction on the weight sequence $\set{\gamma_t}$, which is similar to that in Theorem \ref{thm:gamma_convergence_rates}. We discuss the conditions in detail in Appendix \ref{apdx:oeuvre_clt_proof}. While $V_t$ might not be known in practice, it can be upper-bounded by the recursive variance $\Vup_t$ from Proposition \ref{prop:var_recursion}, which can be used to get asymptotically valid (though more conservative) confidence intervals.

% \todo[inline]{Pick only the most relevant parts of the explanation above. Move the full explanation to the appendix and point the reader there.}

% This would require knowing an upper bound $\sigma_t$, which is possible if the properties of the function class required to bound the uniform stability exactly are known.

\vspace{-0.5em}
\subsection{Fixed-time and Time-uniform Concentration Bounds}
\vspace{-0.5em}
\label{sec:concentration_bounds}
% Since uniform stability results provide bounds on $\max_{z} \abs{\ell_t(z) - \ell_{t-1}(z)}$, they also result in bounds on the size of the updates of the OEUVRE estimator. Using this, we can obtain exponential concentration bounds for $L_t$.
Uniform stability bounds can also be used to establish concentration bounds for $M_t$:
\begin{theorem}
    \label{thm:oeuvre_concentration}
    Let the loss function $\ell$ be bounded in $[0, b]$. Let $\set{\sigma_t}$ be a sequence upper-bounding the uniform stability of the function sequence learned by algorithm $\calA$. Let $\gamma_t$ be defined as in Equation \ref{eq:gamma_optimal} or \ref{eq:gamma_constraint}. Then,
    % Let $\Vup_t$ be defined recursively using $\set{\sigma_t}$ as in Proposition \ref{prop:var_recursion}. Let $\set{\gamma_t}$ be an optimal/non-optimal weight sequence. Then,

    \vspace{-0.5em}
    \begin{enumerate}[label = (\alph*), leftmargin=*]
            \item (Fixed-time bound) For any $\epsilon > 0$ and $t \in \bbN$,
            \vspace{-0.5em}
    \begin{align*}
        \bbP\paren{\abs{M_T} \geq \epsilon} &\leq 2\exp\paren{-\frac{\epsilon^2}{2 \Vup_T}}
    \end{align*}
            \vspace{-0.5em}
            \item (Time-uniform bound) For any $c > 0$ and $T \in \bbN$
            \vspace{-0.5em}
    \begin{align*}
        % \bbP &\paren{\exists t \in [T]: \abs{M_t} \geq \sqrt{\frac{c}{2 V_T}} \cdot \frac{(\Vup_T \Gamma_t^2 + \Vup_t \Gamma_t^2)}{\Gamma_T \Gamma_t}} \\
        % &\quad\quad \leq 2\exp\paren{-c}
        \bbP &\paren{\exists t \in [T]: \abs{M_t} \geq \sqrt{\frac{c}{2 \Vup_T}} \cdot \paren{\frac{\Vup_T \Gamma_t}{\Gamma_T} + \frac{\Vup_t \Gamma_T}{\Gamma_t}}} \\
        &\qquad \leq 2\exp\paren{-c}
    \end{align*}
    \end{enumerate}
    \vspace{-1em}
\end{theorem}
We prove this theorem in Appendix \ref{apdx:oeuvre_concentration_proof}. Part $(a)$ is an application of the Azuma-Hoeffding inequality \cite{wainwright2019high}, and part (b) is applies the master theorem from \cite{howard2020time} to the martingale $(M_{t} / \Gamma_{t})$ and minimizing the bound for the final time step $T$. 
The confidence interval in part (a) has width $\calO(\sqrt{\Vup_t})$ at time $t$, and the confidence sequence in part (b) decreases in size with time, having width $\calO(\sqrt{\Vup_T})$ at the final time step $T$. Similar results can be obtained for other distributional assumptions on the loss using results in \cite{howard2020time}.

\textbf{Interpretation of convergence results.} We note that Theorems \ref{thm:oeuvre_clt} and \ref{thm:oeuvre_concentration} establish the rate of convergence of $L_t$ to the \textit{conditional expectation} of the loss, which is the expected loss of the currently learned function. If the function class $\calG$ for the learning algorithm is such that exact bounds on the uniform stability are known, then Theorems \ref{thm:oeuvre_clt} and \ref{thm:oeuvre_concentration} can be used to construct confidence intervals/sequences for the OEUVRE estimator. On the other hand, if we only know the \textit{order} of uniform stability instead of exact bounds, the above theorems still show that OEUVRE exhibits strong convergence properties, since a stability rate of order $\beta_t$ implies convergence rates and concentration of order $\sqrt{\beta_t}$ from Theorem \ref{thm:gamma_convergence_rates}.

% At time $T$, these confidence intervals/sequences have size $\calO(\sqrt{\Vup_T})$. 
In our running example of a static sequence of functions, this corresponds to a width of 
The extreme instance of static functions $f_t = f$ results in the well-known $\calO(1/\sqrt{T})$ rate for confidence intervals/sequences. The $\calO(1/\sqrt{T})$ rate also holds for learning algorithms with $\calO(1/T)$ stability rates, which holds from FTL with ERM and SGD for strongly convex loss functions with a learning rate decay of the order $1 / t$. On the other hand, algorithms with $\calO(1/\sqrt{T})$ stability rates such as FTRL, OMD, and RDA result in intervals/sequences with width $\calO(1/T^{1/4})$.

% . The convergence rate is closely related to the uniform stability; higher uniform stability implies that past estimates predict the current estimate better, leading to faster convergence.

\vspace{-0.5em}
\section{Adaptive Constant Estimation}
\vspace{-0.5em}
\label{sec:adaptive_const_est}

% The guarantees provided by uniform stability involve constants depending on the quantities such as the diameter of the optimization set, the Lipschitz constants of the function family and loss functions, etc. This can result in a very high multiplicative constant associated with the convergence rate, which would mean that the estimator's variance converges to zero very slowly in practice. This is because these constants hold for worst-case guarantees, whereas the problem instance can be far from the worst case.

OEUVRE requires two types of constants: 1) the upper bound on the variance of loss function evaluations, $b^2$, and 2) the constants for the stability rate $\sigma_t = c r(t)$, where $r(t)$ is the rate of uniform stability and $c > 0$ is the associated constant. In practice these constants are unknown, and we need to know additional properties of the function class and loss functions to estimate them. However, we can estimate these practically without affecting the rate of convergence of the estimator. We begin by showing that the rate of convergence of is not impacted by misspecifying $\sigma_t$ by a constant factor.
\begin{lemma}
  \label{lem:adaptive_const_est}
  Let $\sigma_t = c r(t)$, where $c > 0$. . Let $\Vup_t$ be the recursion from Equation \ref{eq:var_recursion}.
  \vspace{-0.5em}
  % Consider the following recursions:
  \begin{align*}
    \Vup_1 = b^2, &\qquad \Vup_{t} = (\gamma_{t} b + (1 - \gamma_{t}) c r(t))^{2} + \Vup_{t-1}
    \end{align*}
    Let $\Vdag_t$ be the recursion from Equation \ref{eq:var_recursion} with the same weight sequence $(\gamma_t)$ but with misspecified constants $\hat{b}, \hat{c} > 0$ instead of $b$, $c$.
  \vspace{-0.5em}
\begin{align*}
    \Vdag_{1} = \hat{b}^{2}, &\qquad \Vdag_{t} = (\gamma_{t} \hat{b} + (1 - \gamma_{t}) \hat{c} r(t))^{2} + \Vdag_{t-1}
  \end{align*}
  Then, $\Vup_{t} \leq \Vdag_{t} / \paren{\min\set{1, (\hat{c} / c)^{2}} \cdot \min\set{1, (\hat{b} / b)^2}}$.
  
  % Let $c > 0$. Consider the following recursions:
  % \begin{align*}
  %   V_1 = b^2, &\qquad V_{t} = (\gamma_{t} b + (1 - \gamma_{t}) \sigma_{t})^{2} + V_{t-1} \\
  %   V_{1}' = \beta^2 b^{2}, &\qquad V_{t}' = (\gamma_{t} \beta b + (1 - \gamma_{t}) c \sigma_{t})^{2} + V_{t-1}'
  % \end{align*}
  % Then, $V_{t} \leq V_{t'} / (\min\set{1, c^{2}} \cdot \min\set{1, \beta^2})$.
\end{lemma}
\vspace{-0.5em}
We prove this lemma in Appendix \ref{apdx:adaptive_const_est_proof}. We can improve the empirical performance of OEUVRE by approximating the constants using the first $k$ samples, setting $\hat{b}^2$ and $\hat{c}^2$ as the empirical variances over $B = \set{\ell_t(z_t)}_{t=1}^{k}$ and $C = \set{(\ell_t(z_t) - \ell_{t-1}(z_t)) / r_t}_{t=1}^{k}$ respectively.

% as $\hat{c}^2 = \widehat{\Var}(C)$ and $\hat{b}^2 = \widehat{\Var}(B)$, where $C = \set{(\ell_t(z_t) - \ell_{t-1}(z_t)) / r_t}_{t=1}^{k}$, and $B = \set{\ell_t(z_t)}_{t=1}^{k}$. Here, $\widehat{\Var}$ is the empirical variance.

% Let $\Var(\ell_{t}(Z) - \ell_{{t-1}}(Z)\mid \calF_{t-1} \leq \sigma_t^2 = c^2 r^2(t)$ and $\Var(\ell_t(Z)) \leq b^2$. 

% \begin{align*}
% \hat{c} = \Var\paren{\set{\frac{\ell_t(z_t) - \ell_{t-1}(z_t)}{r(t)}}_{t=1}^{k}},
% \end{align*}
% which is the empirical variance of the scaled difference of evaluations. 
 
We can then run OEUVRE initialized at sample $k + 1$ with $\hat{b}$ and $\hat{c}$. We observe empirically that adaptive estimation in this manner can result in significantly better performance, and all our experiments in Section \ref{sec:experiments} employ this method. Empirically, we observe that the first 30 time steps are sufficient to approximate $\hat{c}$ and $\hat{b}$ for good performance of OEUVRE in later time steps. By estimating these hyperparameters adaptively and providing stability order for the algorithm, we specify OEUVRE completely, and there are no other hyperparameters which need to be set.

\vspace{-0.5em}
\section{Experiments}
\vspace{-0.5em}
\label{sec:experiments}

We test the OEUVRE estimator through a series of experiments on diverse tasks consisting of synthetic and real-world datasets. For all experiments, we compare the performance of our estimator against the following baseline methods: 1) Sliding Window (SW), 2) Fading Factor Prequential (FFPreq) \cite{gama2013evaluating}, 3) Exponential Moving Average (EMA), and 4) Adaptive Windowing (ADWIN) \cite{bifet2007learning}. We evaluate the performance of the estimators against the ground truth expected loss using two metrics: the time-averaged Root Mean Squared Error (RMSE) and Mean Absolute Error (MAE) over all time steps. For all experiments, we use the first $t_0=30$ time steps as a burn-in period to initialize the constants for OEUVRE; these time steps are included in the metric calculation. We repeat each experiment over 10 seeded runs.

For each experimental setting, we sweep over a broad range of hyperparameter values for each baseline estimator, choosing the setting with the lowest mean RMSE in hindsight. Calculating the RMSE to choose the best hyperparameter requires knowledge of the true expected loss--an oracle advantage that is unavailable in practice. This results in a very strong benchmark, and we compare the \textit{best setting} from each baseline estimator against OEUVRE with adaptive tuning, which does not need explicit hyperparameter selection. Please refer to Appendix \ref{apdx:implementation_details} for more details.

% While OEUVRE's performance gains appear marginal, we emphasize that baseline hyperparameters were selected via exhaustive grid search to minimize RMSE against the known ground truth loss -- an oracle advantage unavailable in practice. OEUVRE's adaptive tuning eliminates this oracle hyperparameter selection, making even modest improvements over best-in-hindsight baselines practically significant.

% We observe that OEUVRE consistently matches or outperforms baseline methods across all experiments. This is noteworthy because all baseline hyperparameters were selected via exhaustive grid search to minimize RMSE against the known ground truth loss—an oracle advantage unavailable in practice. OEUVRE's parameter-free (or adaptively-tuned) approach achieving comparable or better performance under these conditions demonstrates significant practical value.

% The performance of each estimator is evaluated using three metrics averaged over time: 1) the root mean squared error (RMSE) between the estimated expected loss and the true loss; 2) the mean absolute error (MAE); and 3) the time-averaged bias. All three metrics are computed between the loss estimated by the estimator and the true expected loss $\bbE_{z \sim \calD}[\ell_t(z)]$ (or an approximation of it).

% \subsection{Linear Regression}
\vspace{-1em}
\paragraph{Linear Regression.}
% \paragraph{Data.} We first consider the online ridge regression task. We generate data $\set{(\bfx_t, y_t)}_{t=1}^{T}$ such that $y_t = \bfx^T\mathbf{w}^* + \epsilon_t$, where $\bfx_t \sim \calU[-1, 1]^d$, $\epsilon_t \sim \calN(0, \sigma^2)$ is additive zero-mean Gaussian noise, and $\bfw^*$ is chosen randomly such that $\norm{\bfw}{2} \leq 1$. We set $\sigma \in \set{0.05, 0.1, 0.2}$, $d \in \set{10, 50, 100, 200}$, and run each experiment for $T = 50,000$ time steps.
%
% \paragraph{Learning algorithm.} We use online gradient descent with a learning rate schedule of the form $\gamma / t$, where $\gamma$ is the initial learning rate, which we set to 1. 
%
% \paragraph{Experiment settings.} Our goal is to estimate the expected loss between $\bfw_t$ and $\bfw^*$ by the $L_2$ loss $\bbE_{(x, y) \sim \calD} [(y - \bfx^T \bfw_t)^2]$. For each setting, we take the average across five seeded runs, reporting the average metrics from these runs. The time-averaged metrics are measured starting from time step $d + 1$ and ending at time $T$. We present our results in Figure \ref{fig:linreg}. 
%
We begin with a synthetic online linear regression task, where the data is generated from a linear model with additive Gaussian noise. For the learning algorithm, we use OGD with the learning rate schedule $\eta_0 / \sqrt{t}$, where $\eta_0$ is the initial learning rate. We see in Table \ref{tab:stability_results} that OGD with this learning rate is $1 / \sqrt{t}$-uniformly stable. The goal is to estimate the scaled mean squared error (MSE) of the linear model. We run each experiment for $T=10,000$ time steps, varying both the number of dimensions and the noise level across experiments. We present our results in Figure \ref{fig:expt_linreg_results}, where we compare OEUVRE with the best baseline estimator and the median baseline across all methods and hyperparameters.

\begin{figure*}[htbp]
  \centering
  \includegraphics[width = 0.7\textwidth]{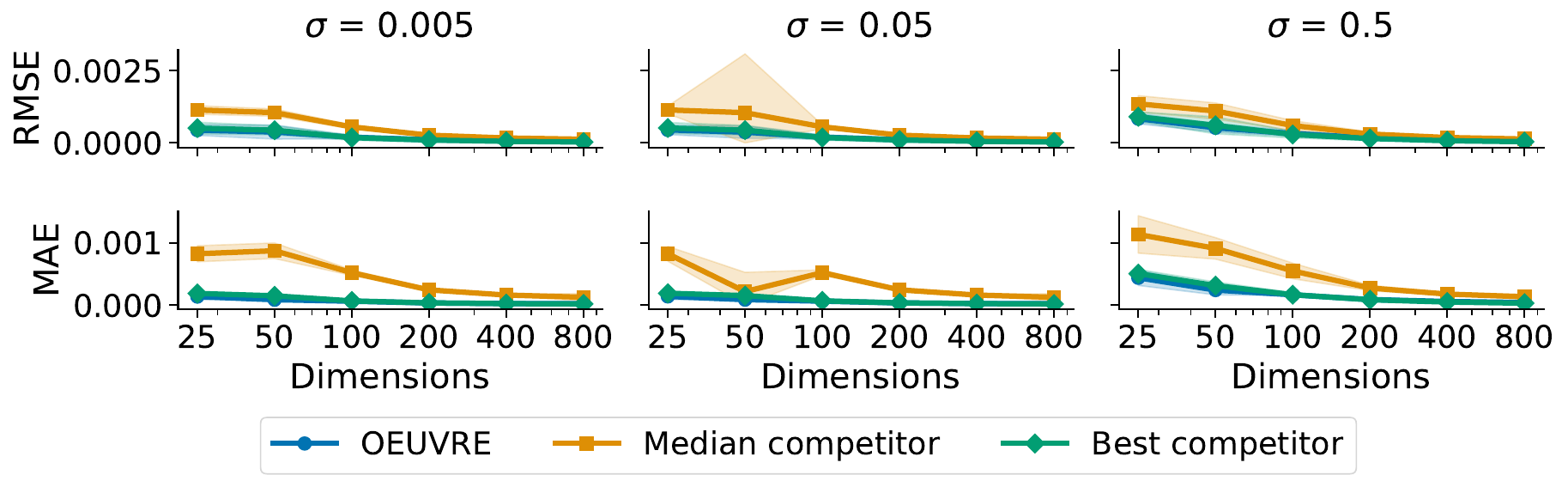}
  % \caption{Performance comparison of different estimators on the online ridge regression task.}
  \caption{Performance comparison of OEUVRE against the best baseline and the median baseline for the online linear regression task. OEUVRE achieves competitive RMSE and bias when compared to the best baseline.}
  \label{fig:expt_linreg_results}
  \vspace{-1em}
\end{figure*}

We observe that OEUVRE achieves competitive mean RMSE when compared against the best baseline across dimensions and noise levels. While the difference in performance is not large, it is significant given that the competitor is chosen from broad hyperparameter sweeps of four different baseline methods. The MAE for OEUVRE is also very close to the best baseline, being significantly lower than the median baseline for smaller dimensions ($d = 25-100$).

% These results indicate that OEUVRE with adaptively set hyperparameters is a robust choice for this setting.

% We observe that the OEUVRE estimator provides the lowest RMSE and MAE across almost all settings while also maintaining the lowest bias. The only setting where a different estimator performs better is for $\sigma = 0.01$, where the sliding window estimator with a window size of 10 performs slightly better than OEUVRE for some dimensions. However, the adaptive nature of OEUVRE ensures that it performs competitively across a broad range of dimensions and noise levels.

% While the sliding window estimator with a window size of 10 performs better than OEUVRE for $\sigma = 0.01$, the adaptive nature of the OEUVRE estimator means that it performs competitively even in this setting, outperforming the sliding window estimator in other settings.

% \clearpage
%
% \subsection{Classification with Logistic Regression}
% \begin{figure*}[h]
%   \centering
%
%   \includegraphics[width = 0.7\textwidth]{./expt_results/real_data/diabetes_health_indicators.pdf}
%   \caption{Performance comparison of different estimators on the Diabetes Health Indicators dataset.}
%   \label{fig:diabetes_health_indicators}
% \end{figure*}
%
% \input{./expt_results/real_data/diabetes_health_indicators_shuffle_True_results.tex}
%
% \clearpage

% \subsection{Prediction with Expert Advice}
\vspace{-1em}
\paragraph{Prediction with Expert Advice.}
% \paragraph{Data.} We consider a synthetic task where we have $k$ experts, each making a prediction at all time steps. The loss incurred by expert $i$ at time $t$ is given bI $\ell_{t,i} \in [0, 1]$, and is drawn in an iid manner from distribution $\calD_i$. We conduct experiments over two family of distributions: 1) Bernoulli distribution with mean $p_i$ for expert $i$ drawn uniformly at random from $[0, 1]$, and 2) Beta distribution with parameters $(\alpha_i, \beta_i)$ for expert $i$ drawn uniformly at random from $[1, k]$. We consider $k \in \set{10, 50, 100}$ experts and run each experiment for $T = 10^4$ time steps.
%
% \paragraph{Learning algorithm.} We use the Hedge algorithm with learning rate given by $\eta = \log(K) / \sqrt{T}$. We know that the Hedge algorithm is uniformly stable with rate $\calO(1 / \sqrt{T})$ since it is an instance of online mirror descent [CITE]. We denote the distribution over the experts maintained by the algorithm at time $t$ by $\bfp_t \in \Delta_{k - 1}$.
%
% \paragraph{Experiment settings.} Our goal is to estimate the expected loss of the current distribution maintained by Hedge, i.e., $\sum_{i=1}^k p_{t,i} \bbE[\ell_{t,i}]$. For each setting, we take the average across five seeded runs, reporting the average metrics from these runs. The time-averaged metrics are measured starting from time step $d + 1$ and ending at time $T$. We present our results in Figure \ref{fig:hedge}.
%
We then consider the prediction with expert advice task, where we have $K$ experts, each making a prediction at all time steps. We draw the loss incurred by each expert from either a Bernoulli or a Beta distribution in an i.i.d. manner. We use the Hedge algorithm \cite{mourtada2019optimality} with a learning rate of $\eta = \sqrt{\log(K) / t}$ to learn the distribution over experts. From Table \ref{tab:stability_results}, we see that this algorithm is $1/\sqrt{t}$-uniformly stable. The goal is to estimate the expected loss of the currently learned distribution over experts. We run each experiment for $T = 10,000$ time steps, varying both the number of experts and the loss distributions across experiments. We compare OEUVRE with the best baseline estimator and the median baseline across all methods and hyperparameters in Figure \ref{fig:expt_hedge_results}.

% \begin{figure*}[h]
%   \centering
%   \includegraphics[width = 0.7\textwidth]{./paper_plots/expert_advice.pdf}
%   \caption{Performance comparison of different estimators on the Hedge algorithm with expert advice task.}
%   \label{fig:hedge}
% \end{figure*}

\begin{figure*}[htbp]
    \centering
    \begin{subfigure}[t]{0.48\textwidth}
        \centering
        \includegraphics[width=\textwidth]{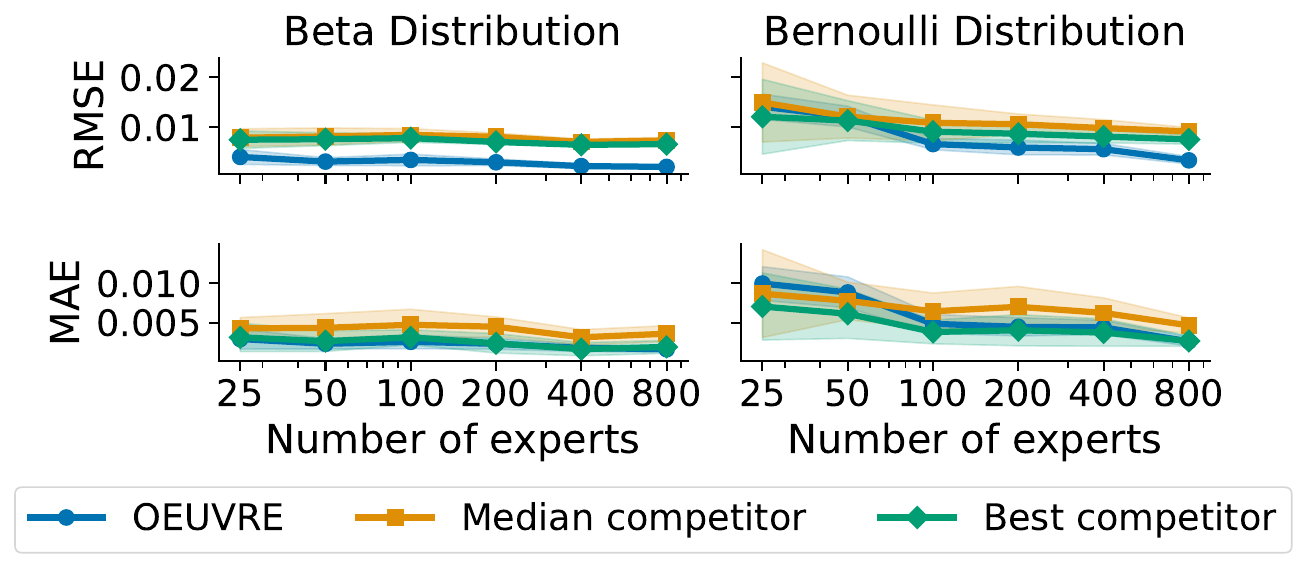}
        \vfill
        % \caption{Results for prediction with expert advice}
        \caption{}
        \label{fig:expt_hedge_results}
    \end{subfigure}
    \hfill % This command adds horizontal space between the figures
    \begin{subfigure}[t]{0.48\textwidth}
        \centering
        \includegraphics[width=\textwidth]{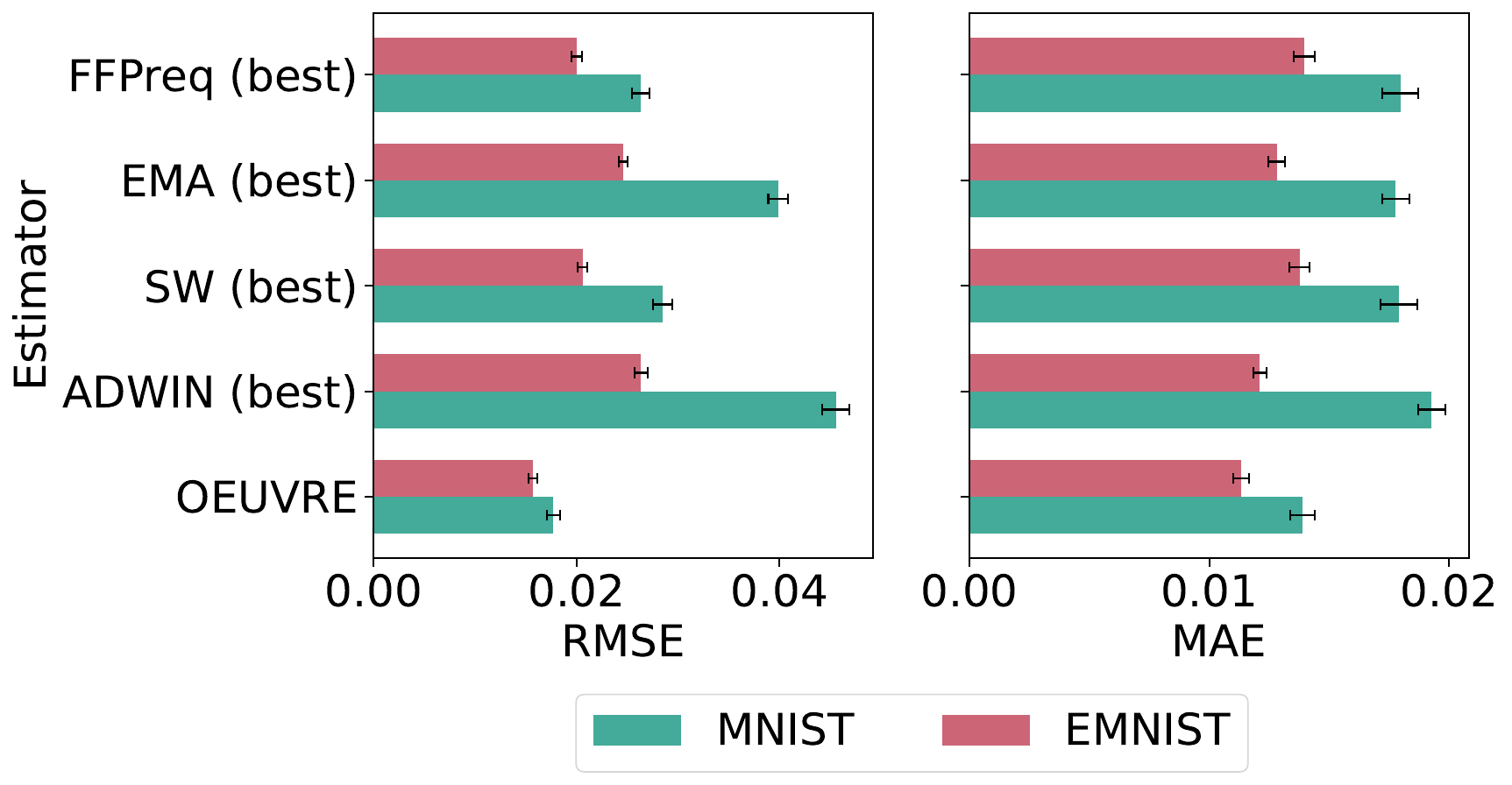}
        % \caption{Results for neural networks}
        \caption{}
        \label{fig:expt_neural_networks_result}
    \end{subfigure}
    \caption{Comparison of OEUVRE with best-performing baseline methods for (a) prediction with expert advice task and (b) neural networks task. In (a), it achieves smaller RMSE and comparable MAE when compared to the best baseline. In (b), it achieves the lowest RMSE and MAE for both datasets.}
    \label{fig:performance_comparison}
    \vspace{-1.2em}
\end{figure*}

We observe that for Beta-distributed losses, the OEUVRE estimator achieves significantly lower RMSE when compared to the best baseline, while achieving comparable MAE. For Bernoulli-distributed losses, the RMSEs are comparable for smaller dimensions $(d = 25, 50)$, after which OEUVRE outperforms the best baseline. The MAE for OEUVRE is closer to the median baseline for smaller dimensions, becoming comparable to the best baseline for larger dimensions.

% While the RMSEs are comparable for smaller dimensions for Bernoulli-distributed losses, OEUVRE achieves lower RMSEs for larger dimensions. The bias for OEUVRE is comparable to that of the best competitor, the former often being slightly lower than the latter.

% We observe that OEUVRE consistently provides the lowest RMSE and bias across all settings.
% \todo[inline]{Need more interpretation}

% \clearpage

% \subsection{Neural Networks}
\vspace{-1em}
\paragraph{Neural Networks.}
Having validated OEUVRE on tasks with known stability bounds, we now stress-test whether its properties extend to settings beyond our theoretical coverage. We consider two popular image classification tasks, MNIST \cite{deng2012mnist} and EMNIST-digits \cite{cohen2017emnist}, and train a simple neural network on them with AdamW. The goal is to estimate the expected cross-entropy (CE) loss on the test set, which we approximate by evaluating the model on 1024 randomly drawn test samples at each time step. To update the estimators at each time step, we evaluate the model on 128 randomly drawn test samples and calculate the mean CE loss. We run each experiment for 10 epochs.

Though stability results are difficult to establish for neural networks with adaptive learning rates, we include these experiments to demonstrate OEUVRE's practical robustness beyond its theoretical guarantees. We approximate $\sigma_t$ as $C' \norm{\eta_t \odot g_t}{\infty}$, where $\eta_t$ is a vector of the learning rate per neuron, and $g_t$ is the gradient vector. This heuristic proxy is guided by the intuition that, assuming the neural network is $C$-Lipschitz, $\abs{\ell_t - \ell_{t-1}} \leq C \norm{w_t - w_{t-1}}{2} \leq C' \norm{\eta_t \odot g_t}{\infty}$, where $w_t$ is the vector of neuron weights at time $t$. Thus, the $L^{\infty}$ norm acts as a bound on the extent of function updates. We present our results in Figure \ref{fig:expt_neural_networks_result}.

We observe that for both MNIST and EMNIST-digits, OEUVRE achieves lower mean RMSE and MAE than the competing methods. In Figure \ref{fig:apdx_neural_networks_batchsz_128_bias}, we observe that the time-averaged bias of OEUVRE is significantly lower than those of the competitors, indicating that its estimate follows the true expected loss more closely. However, the variance across runs for the bias for MNIST is high, indicating that the heuristic $\sigma_t$ approximation may be inaccurate for adaptive optimizers. Further investigation of stability-based tuning for adaptive learning rates remains important future work. We observe similar results when reducing the number of test samples shown to the estimators at each time step from 128 to 8 (Figure \ref{fig:nn_batchsz_8}). 

% However, OEUVRE's estimates but exhibits high variance on MNIST, suggesting 

% (with the constant estimated adaptively as in Section \ref{sec:adaptive_const_est}). 

% The estimators approximate the expected cross-entropy loss of the neural network on the test dataset, which is approximated by a drawing 1024 test samples at each time step. To update their estimates, the estimators are given CE losses for 128 time steps. We run each experiment for 10 epochs.
% \begin{figure*}[h]
%   \centering
%   \includegraphics[width = 0.7\textwidth]{./paper_plots/neural_networks.pdf}
%   \caption{Performance comparison of different estimators on simple neural networks trained with AdamW.}
%   \label{fig:neural_networks}
% \end{figure*}

\vspace{-1.2em}
\section{Discussion}
\vspace{-1.2em}

Our work suggests several directions for future research. The theoretical guarantees for OEUVRE currently assume i.i.d. data; we conduct some experiments on datasets with drift in Appendix \ref{apdx:non_stationary_data_expts}, and OEUVRE performs reasonably well for gradual and seasonal drift, with performance degrading in cases with abrupt drift. Extending OEUVRE to non-stationary settings via mixing assumptions on the data or combining OEUVRE with windowing is a key next step.

% ; although we extending these to non-stationary settings, 

% perhaps by designing a hybrid estimator with drift detection, is a key next step. 

Our analysis also requires pre-determined weight sequences, excluding algorithms with adaptive learning rates like Adam, though our experiments suggest heuristic approximations can be effective in practice. Finally, our framework relies on convenient $L^{\infty}$ uniform stability bounds, but could immediately benefit from the development of tighter, systematic $L^2$ stability theory for online learning, which would directly improve our variance bounds and convergence rates.

\subsubsection*{Acknowledgement}
This material is based upon work supported by the AI Research Institutes Program funded by the National Science Foundation under the AI Institute for Societal Decision Making (NSF AI-SDM), Award No. 2229881.

\printbibliography

\clearpage

\clearpage
\onecolumn
\raggedbottom
\appendix
\aistatstitle{Supplement for OEUVRE: OnlinE Unbiased Variance-Reduced loss Estimation}

\section{OEUVRE with Batching}
\label{sec:oeuvre_batching}

The OEUVRE estimator can be modified to be used in batch settings, where the learner updates on a batch of $B$ samples at every time step. This could be the case when training is expensive, which makes it preferable when enough samples have been gathered. It might also be that the distribution of samples over time is non-uniform, with most samples arriving in rapid succession and with large gaps between these spikes. In such situations, training might be done between the spikes of samples.

We assume that at each time step $t$, a batch of $B$ i.i.d. samples $Z_t = \set{z_{t1}, \ldots, z_{tB}}$ arrives and is used to update the learned algorithm. Function $f_t$ is trained on the past set of samples $Z_1 \cup \ldots \cup Z_{t-1}$, and we define $\ell_t = \ell \circ f_t$ as before. Our modified OEUVRE estimator is
\begin{align*}
  L_t &= \frac{1}{B} \sum_{i=1}^{B} \ell_t(z_{ti}) + (1 - \gamma_t) \sqbrac{L_{t-1} - \frac{1}{B} \sum_{i=1}^{B} \ell_{t-1}(z_{ti})}.
\end{align*}
We note that since $Z_t$ is evaluated on $f_t$ and $f_{t-1}$, we need to maintain a cache of size $B$ containing the future samples. Since our learned function is trained on one time step less than those available, the learning algorithms incurs only an additively higher cumulative regret at the benefit of more accurate estimation of their performance. 

For most common algorithms (like FTL with ERM, FTL, and OMD), the batched versions of these algorithms have the same uniform stability for time step $t$ as the single-sample version. Thus, the weight sequence $\set{\gamma_t}$ is identical in both cases. By averaging our loss evaluations over $B$ i.i.d. samples, the variance of the estimator reduces by a factor of $B$. This reduction in variance by batching is also exhibited by other evaluation methods like prequential evaluation, sliding windown estimation and exponential moving average. We also observe that the batches need not be of the same size; the calculation of $\gamma_t$ only depends on uniform stability guarantees which do not get worse with batching, and OEUVRE can be used with dynamic batch sizes while achieving similar convergence guarantees. 

\section{Deferred proofs}

\subsection{Proof of Proposition \ref{prop:mtgle_prop}}
\label{apdx:mtgle_prop_proof}

\begin{proof}
  We define $\Delta_t(x, y) = \ell_t(x, y) - \bbE[\ell_t(x, y) \mid \calF_{t-1}]$. We then have
  \begin{align*}
    \bbE[\Delta_t(x, y) \mid \calF_{t-1}] &= \bbE[\ell_t(x, y) \mid \calF_{t-1}] - \bbE[\ell_t(x, y) \mid \calF_{t-1}] \\
    &= 0
  \end{align*}
  Now, $M_t$ can be expressed as
  \begin{align*}
    M_t &= L_t - \bbE[\ell_t(x, y) \mid \calF_{t-1}] \\
    &= \ell_t(x_t, y_t) - \bbE[\ell_t(x, y) \mid \calF_{t-1}] - (1 - \gamma_t) \sqbrac{\ell_{t-1}(x_t, y_t) - \bbE[\ell_{t-1}(x, y) \mid \calF_{t-1}]} \\
    &\quad + (1 - \gamma_t) \sqbrac{L_{t-1} - \bbE[\ell_{t-1}(x, y) \mid \calF_{t-1}]} \\
    &= \Delta_t(x_t, y_t) - (1 - \gamma_t) \Delta_{t-1}(x_t, y_t) + (1 - \gamma_t) M_{t-1} \\
  \end{align*}
  Dividing throughout by $\Gamma_t$, we have
  \begin{align*}
    \frac{M_t}{\Gamma_t} &= \frac{\Delta_t(x_t, y_t)}{\Gamma_t} - \frac{\Delta_{t-1}(x_t, y_t)}{\Gamma_{t-1}} + \frac{M_{t-1}}{\Gamma_{t-1}}
  \end{align*}
  Taking the conditional expectation on both sides, we get
  \begin{align*}
    \bbE\sqbrac{\frac{M_t}{\Gamma_t} \mid \calF_{t-1}} &= \bbE\sqbrac{\frac{\Delta_t(x_t, y_t)}{\Gamma_t} \mid \calF_{t-1}} - \bbE\sqbrac{\frac{\Delta_{t-1}(x_t, y_t)}{\Gamma_{t-1}} \mid \calF_{t-1}} + \bbE\sqbrac{\frac{M_{t-1}}{\Gamma_{t-1}} \mid \calF_{t-1}} \\
    &= \frac{M_{t-1}}{\Gamma_{t-1}}
  \end{align*}
  Thus, $(M_t / \Gamma_t)$ is a martingale. Moreover, we have
  \begin{align*}
    \bbE\sqbrac{\frac{M_1}{\Gamma_1} \mid \calF_0} &= \bbE[L_1 - \bbE[\ell_1(x, y)]] \\
    &= \bbE[\ell_1(x, y)] - \bbE[\ell_1(x, y)] = 0
  \end{align*}
\end{proof}

\subsection{Proof of Proposition \ref{prop:var_recursion}}
\label{apdx:var_recursion_proof}

\begin{proof}
  We have
  \begin{align*}
    \Var(M_t) &= \Var\paren{\Delta_t(x_t, y_t) + (1 - \gamma_t)[M_{t-1} - \Delta_{t-1}(x_t, y_t)]} \\
    &= \Var\paren{\Delta_t(x_t, y_t) - (1 - \gamma_t) \Delta_{t-1}(x_t, y_t) + (1 - \gamma_t) M_{t-1}} \\
    &= \Var\paren{\Delta_t(x_t, y_t) - (1 - \gamma_t) \Delta_{t-1}(x_t, y_t)} + (1 - \gamma_t)^2 \Var(M_{t-1}) + 2 \bbE\sqbrac{\paren{\Delta_{t}(x_t) - (1 - \gamma_t) \Delta_{t-1}(x_t)} M_{t-1}}
  \end{align*}
  The third time can be simplified as
  \begin{align*}
    &2 \bbE\sqbrac{\bbE\sqbrac{\paren{\Delta_{t}(x_t) - (1 - \gamma_t) \Delta_{t-1}(x_t)} M_{t-1}} \mid \calF_{t-1}} \\
    &\quad= 2 \bbE\sqbrac{M_{t-1} \bbE\sqbrac{\paren{\Delta_{t}(x_t) - (1 - \gamma_t) \Delta_{t-1}(x_t)}} \mid \calF_{t-1}} \\
    &\quad\stackrel{(i)}{=} 0,
  \end{align*}
  where $(i)$ comes from $\bbE[\Delta_t(x_t) \mid \calF_{t-1}] = 0$ and $\bbE{\Delta_{t-1} \mid \calF_{t-1}} = 0$. Thus, our variance simplifies to
  \begin{align*}
    \Var(M_t) &= \Var\paren{\Delta_t(x_t, y_t) - (1 - \gamma_t) \Delta_{t-1}(x_t, y_t)} + (1 - \gamma_t)^2 \Var(M_{t-1}) \\
    &\leq (\sqrt{\Var(\gamma_t \Delta_t(x_t, y_t))} + \sqrt{\Var((1 - \gamma_t)(\Delta_t(x_t, y_t) - \Delta_{t-1}(x_t, y_t)))})^2 + (1 - \gamma_t)^2 \Var(M_{t-1}) \\
    &\leq (\gamma_t b + (1 - \gamma_t) \sigma_t)^2 + (1 - \gamma_t)^2 \Var(M_{t-1}) \\
  \end{align*}
\end{proof}

\subsection{Proof of Lemma \ref{lem:opt_gamma_consistency}}
\label{apdx:opt_gamma_consistency_proof}

\begin{proof}
  We first show that $V_t \to 0$, followed by $\gamma_t \to 0$.
  \paragraph{Showing $V_t \to 0$.} Since $\sigma_t \to 0$, there is some $t_0(\epsilon)$ such that for all $t \geq t_0(\epsilon)$, $ \sigma_t < \epsilon$. We now show that there should exist some $t_V(\epsilon) \geq t_0(\epsilon)$ such that, for all $t \geq t_V(\epsilon)$, $V_t < 2 b \epsilon$. We divide our analysis into three cases:

  \begin{enumerate}[label = (\alph*)]
    \item Case 1, $V_{t - 1} \leq \sigma_t(b - \sigma_t)$: In this case,
      \begin{align*}
        V_t \leq V_{t - 1} + \sigma_t^2 \leq b \sigma_t < 2 b \epsilon
      \end{align*}

    \item Case 2, $\sigma_t > b$: In this case,
      \begin{align*}
        V_t \leq b^2 < b \sigma_t < 2 b \epsilon
      \end{align*}

    \item Case 3, none of the above: For all $t \geq t_0(\epsilon)$, we also have
      \begin{equation*}
        \begin{aligned}
          V_{t} \leq \frac{b^2 V_{t - 1}}{V_{t - 1} + (b - \sigma_t)^2} \leq \frac{b^2 V_{t - 1}}{V_{t - 1} + (b - \epsilon)^2}
        \end{aligned}
      \end{equation*}
      Thus, the third term is an upper bound on $V_t$. Now, consider another recursion $W_t$, $t \geq t_0(\epsilon)$, defined as
      \begin{equation*}
        \begin{aligned}
          W_{t} = \begin{cases} V_{t_0(\epsilon)}, \quad& t = t_0(\epsilon) \\
          \frac{b^2 W_{t - 1}}{W_{t - 1} + (b - \epsilon)^2}, \quad& \text{otherwise} \end{cases}
        \end{aligned}
      \end{equation*}
      We have $W_{t_0(\epsilon)} = V_{t_0(\epsilon)}$, and $W_{t_0(\epsilon) + 1} \geq V_{t_0(\epsilon) + 1}$. Moreover, the following function $f(x)$ is strictly increasing in $x$:
      \begin{align*}
        f(x) = \frac{b^2 x}{x + (b - \epsilon)^2}
      \end{align*}
      Thus, by induction, $W_t \geq V_t \forall t \geq t_0(\epsilon)$. Thus, if we show that there is some $t_V(\epsilon)$ such that for all $t \geq t_V(\epsilon)$, $W_t < b \epsilon$, then $V_t < b \epsilon$ for all $t \geq t_V(\epsilon)$.

      Now, the definition of $W_t$ is of the nature of a fixed point recursion in $W_t$. We calculate the fixed points of this recursion:
      \begin{equation*}
        \begin{aligned}
          x &= \frac{b^2 x}{x + (b - \epsilon)^2} \\
          x(x - (b^2 - (b - \epsilon)^2)) &= 0 \\
          x(x - \epsilon(2b - \epsilon)) &= 0 \\
        \end{aligned}
      \end{equation*}
      Thus, the two fixed points of the recursion are $0$ and $\epsilon(2b - \epsilon)$.

      For all $x \in (0, \epsilon(2b - \epsilon))$, we have $x < f(x)$. Thus, $x$ is mapped to a value higher than itself. For all $x > \epsilon(2b - \epsilon)$, we have $x > f(x)$. Thus, $x$ is mapped to a value lower than itself. This indicates that $\epsilon(2b - \epsilon)$ is the 'attracting' fixed point for $f(x)$. Thus, $W_t \to \epsilon(2b - \epsilon)$.

      This means that there should be some $t_V(\epsilon)$ such that for all $t \geq t_V(\epsilon), W_t < 2b \epsilon$. This concludes the case-wise analysis.
  \end{enumerate}

  Over our three cases, our choice of $\epsilon$ was arbitrary. Thus, we should have $V_t \to 0$.

  \paragraph{Showing $\gamma_t \to 0$.} We observe the recursion for $\gamma_t^*$ in the non-trivial case, i.e., when $\gamma_t \not\in \set{0, 1}$:
  \begin{align*}
    \gamma_t^* &= \frac{V_{t-1} - \sigma_t(b - \sigma_t)}{V_{t-1} + (b - \sigma_t)^2}
  \end{align*}
  Now, as $\sigma_t \to 0$ and $V_t \to 0$, we observe that the numerator tends to zero, and the denominator tends to $b^2$. Thus, $\gamma_t^* \to 0$.
\end{proof}

\subsection{Proof of Theorem \ref{thm:gamma_convergence_rates}}
\label{apdx:gamma_convergence_rates_proof}

We prove this result by establishing a result for non-optimal choices for $\gamma_t$

  \paragraph{Non-optimal choices of $\gamma_t$.} Although the optimal $\gamma_t^*$ sequence results in the lowest variance, it might be difficult to establish the convergence rate for $\Vup_t$ in closed form. The following lemma generalizes the consistency result for $L_t$ for other choices of $\gamma_t$ along with providing an upper bound on the convergence rate.
  \begin{lemma}
    \label{lem:non_opt_gamma_consistency}
    Let $\set{\gamma_t}$ be a weight sequence such that there exists some $t_0 \geq 1$ with $\gamma_t = \Omega(\sigma_t)$, and $1 / \gamma_t - 1 / \gamma_{t-1} \leq 1$ for all $t \geq t_0$. Then $\Gamma_{t}^{2} \Var(M_t) \leq \Vup_t = \calO(\gamma_t)$.
  \end{lemma}

  Our choice of $\gamma_t$ in Theorem \ref{thm:gamma_convergence_rates} ensures that the two conditions on $\gamma_t$ are satisfied.
  % The constraint on $\set{\gamma_t}$ informally means that the rate of decay of $\gamma_t$ should eventually be $\Omega(1/t)$.

  % The constraint on $(\gamma_{t})$ is quite general, being satisfied by $1/t$ and polynomial decays of the form $c t^{-\alpha}$, $\alpha > 1$. These are realistic rates for uniform stability (Section \ref{sec:known_stability_results}), and the above lemma provides us with an explicit bound on the variance of our estimator in such cases. This lemma also formalizes the intuition that a slower rate of decay for $\gamma_t$ results in slower convergence, since less information is carried over from past estimates if $\gamma_t$ is larger than optimal.

  % Finally, Lemma \ref{lem:non_opt_gamma_consistency} helps us get a (possibly loose) bound on $\Var(M_t)$ with the optimal choice of $\gamma_t^*$ -- since this choice has the lowest variance, the variance with choice $\gamma_{t } = \sigma_{t}$ should be greater. Thus, we should have $\Gamma_{t}^{2} \Var(M_t) = \calO(\sigma_t)$. This provides us with a direct and simple upper bound on the variance of the OEUVRE estimator in terms of $\sigma_{t}$. This bound also holds for the sum of conditional variances as in Proposition \ref{prop:var_recursion}.

  % Thus, even if $\sigma_t$'s convergence rate is faster than $1/t$, $\gamma_t$ should still eventually decay as $1/t$, which results in the OEUVRE estimator looking similar to (but not necessarily exactly like) the empirical mean estimator.

\begin{proof}
  We prove the result for $t_0 = 1$ first. We extend the weight sequence to $t = 1$ by letting  $\gamma_1 = 1$. $c = \sup_{t \geq 1} c_t / \gamma_t$. We prove that $V_t \leq (b + c)^2 \gamma_t$ by induction. For t = 1, we have $V_1 = \Var(\Delta_1) \leq b^2 \leq (b + c)^2 = (b + c)^2 \gamma_1^2$. For $t > 1$, we have
  \begin{align*}
    1 &\stackrel{(i)}{\geq} \frac{1}{\gamma_t} - \frac{1}{\gamma_{t-1}} \\
    \gamma_t \gamma_{t-1} &\geq \gamma_{t-1} - \gamma_t \\
    \gamma_t &\geq \gamma_{t-1} (1 - \gamma_t) \\
    \gamma_t (1 - \gamma_t) &\geq \gamma_{t-1} (1 - \gamma_t)^2 \\
    \gamma_t &\geq \gamma_t^2 + (1 - \gamma_t)^2 \gamma_{t-1} \\
    (b + c)^2 \gamma_t &\geq (b + c^2) \gamma_t^2 + (1 - \gamma_t)^2 (b + c)^2 \gamma_{t-1} \\
    &\geq (b \gamma_t + c \gamma_t)^2 + (1 - \gamma_t)^2 (b + c)^2 \gamma_{t-1} \\
    &\geq (b \gamma_t + \sigma_t)^2 + (1 - \gamma_t)^2 (b + c)^2 \gamma_{t-1} \\
    &\geq (b \gamma_t + (1 - \gamma_t) \sigma_t)^2 + (1 - \gamma_t)^2 (b + c)^2 \gamma_{t-1} \\
    &\stackrel{(ii)}{\geq} (b \gamma_t + (1 - \gamma_t) \sigma_t)^2 + (1 - \gamma_t)^2 V_{t-1} \\
    &= V_t,
  \end{align*}
  where $(i)$ is our assumption on $\set{\gamma_t}$, and $(ii)$ comes from the induction hypothesis. Thus, we have $V_t \leq (b + c)^2 \gamma_t$, which proves our result for $t_0 = 1$.

  We can extend the result for $t_0 > 1$ by letting
  \begin{align*}
    c = \max\set{\sup_{t \geq t_0} \sigma_t / \gamma_t, \sqrt{\frac{V_{t_0}}{\gamma_{t_0}}} - b}.
  \end{align*}
  This would mean that at time step $t_0$, we would have $V_{t_0} \leq (b + c)^2 \gamma_{t_0}$, along with $\sigma_t \leq c \gamma_t$ for all $t \geq t_0$. We can then do induction on $t$ starting from $t_0$ as done for $t_0 = 1$.  
\end{proof}

\subsection{Proof of Theorem \ref{thm:oeuvre_clt}}
\label{apdx:oeuvre_clt_proof}
\paragraph{Elaborating on conditions.} We discuss the conditions in detail here. Condition $(i)$ assumes that the conditional variance of the loss at time $t$ is both upper and lower bounded by positive constants. While the upper bound is something we assume originally, the lower bound typically holds in non-realizable settings. In realizable settings, it can be made to hold by inserting a small amount of white noise in the labels $y_t$, which would lead to similar convergence rates without impacting the estimator performance. 

Condition $(ii)$ states that the conditional variances converge in probability to some random variable $v^*$. This condition is satisfied by strongly convex loss landscapes, where the sequence of loss functions converge in probability to the optimal loss function. However, the condition we state is more general and can hold for non-convex settings, where the sequence of loss functions can be attracted to multiple local optima. Condition $(iii)$ is a restriction on the kurtosis of the update of the OEUVRE estimator. Importantly, the kurtosis is total, i.e., it is calculated over the entire data sequence $\set{Z_i}_{i \leq t}$. The bound on the kurtosis need not be known, and this condition should be satisfied in most practical settings with light tails for the losses. Condition $(iv)$ is a restriction on $\set{\gamma_t}$ similar to that in Lemma \ref{lem:non_opt_gamma_consistency}, with the chief difference being $\gamma_t$ having to decay strictly slower than $\sigma_t$.

\paragraph{Proof of \ref{thm:oeuvre_clt}.} 
We define a triangular array $\set{U_{n,i}}$ of normalized martingale differences, with $U_{n,1} = M_1 / \Vbar_n$, and
\begin{align*}
  U_{n,i} = \sqbrac{\frac{M_i}{\Gamma_i} - \frac{M_{i-1}}{\Gamma_{i-1}}} \frac{1}{\sqrt{\Vbar_n}} = \sqbrac{\frac{\Delta_i(x_i, y_i)}{\Gamma_i} - \frac{\Delta_{i-1}(x_i, y_i)}{\Gamma_{i-1}}} \frac{1}{\sqrt{\Vbar_n}},
\end{align*}
where $\Delta_i(\cdot) = \ell_t(\cdot) - \bbE[\ell_t(x, y) \mid \calF_{t-1}$, and $\Vbar_n = \Var(M_n / \Gamma_n)$. Since $U_{n,i}$ is a martingale difference sequence for each $n$,
\begin{align*}
  \Var(\sum_{i=1}^{n} U_{n,i}) = \Var\paren{\frac{M_n}{\Gamma_n} \cdot \frac{1}{\sqrt{\Vbar_n}}} = 1,
\end{align*}
which means that each $U_{n,i}$ is normalized to have variance 1.

We also define $\Lambda_1 = M_1 / \Gamma_1 = M_1$, and
  \begin{align*}
    \Lambda_i = U_{n,i} \cdot \sqrt{\Vbar_n} = \sqbrac{\frac{\Delta_i(x_i)}{\Gamma_i} - \frac{\Delta_{i-1}(x_i)}{\Gamma_{i-1}}},
  \end{align*}
  which is the un-normalized martingale difference. Clearly, we have
  \begin{align*}
    M_n = \sum_{i = 1}^{n} \Lambda_i \implies \Vbar_n = \Var(M_n) = \sum_{i = 1}^{n} \Var(\Lambda_i).
  \end{align*}

We now state three lemmas which will help us prove the theorem.

\begin{lemma}
  \label{lem:single_term_total_var_limit}
  Let $\Var(\Delta_i(x, y)) \to \vbar^*$ and $\Var(\ell(x, y)) \in [\nu^2, b^2]$ for all $f \in \calG$. Then, 
  \begin{enumerate}
    \item $\lim\inf_i \Var(\Lambda_i) \geq \gamma_1 \nu^2 > 0$.
    \item $\Var(\Lambda_i) / \Var(\Lambda_{i-1}) \to 1$.
  \end{enumerate}
\end{lemma}

\begin{lemma}
  \label{lem:total_vars_sums_ratio}
  Let $\set{A_i}_{i \in \bbN^+}$ be a sequence of positive real numbers such that 1) $A_i / A_{i-1} \to 1$, and 2) there exists some $m > 0$ such that $\lim\inf_i A_i \geq m$. Then,
  \begin{align*}
    R_n = \frac{\sum_{i = 1}^{n} A_i^2}{\paren{\sum_{i = 1}^{n} A_i}^2} \to 0
  \end{align*}
\end{lemma}

\begin{lemma}
  \label{lem:var_sum_ratio_limit}
  Let $\Var(\Delta_i(x_i, y_i) \mid \calF_{i-1}) \stackrel{\text{p}}{\to} v^*$, where $v^*$ is a random variable with domain $[\nu^2, b^2]$. Then,
  \begin{align*}
    \frac{V_n}{\Vbar_n} = \frac{\sum_{i=1}^{n} v_i}{\sum_{i=1}^{n} \vbar_i} \stackrel{\text{p}}{\to} \frac{v^*}{\vbar^*},
  \end{align*}
    where $V_n = \sum_{i=1}^{n} \Var(M_i / \Gamma_i \mid \calF_{i-1}) = \sum_{i=1}^{n} \Var(\Lambda_i / \Gamma_i \mid \calF_{i-1})$, and $V_n = \sum_{i=1}^{n} \Var(U_{n,i} \mid \calF_{i-1}) = \sum_{i=1}^{n} \Var(\Lambda_i \mid \calF_{i-1})$.
\end{lemma}

We can now prove the original theorem.
\begin{proof}[Proof of Theorem \ref{thm:oeuvre_clt}]

  We prove the following two conditions which are sufficient for the martingale CLT to hold \cite{hall2014martingale}:
  \begin{enumerate}
    \item \textit{Lindeberg condition}: For all $\epsilon > 0$, $\sum_{i=1}^{n} \bbE[U_{n,i}^2 \bbI(\abs{U_{n,i}} > \epsilon)] \to 0$.

      For each term in the Lindeberg condition, we have the following inequality:
      \begin{equation*}
        \begin{aligned}
          \bbE[U_{n,i}^2 \bbI(\abs{U_{n,i}} > \epsilon)] &\stackrel{(i)}{\leq} \sqrt{\bbE[U_{n,i}^4]} \cdot \sqrt{\bbE[\bbI(\abs{U_{n,i}} > \epsilon)^2]} \\
          &\leq \sqrt{\bbE[U_{n,i}^4]} \cdot \sqrt{\bbE[\bbI(U_{n,i}^4 > \epsilon^4)]} \\
          &\leq \sqrt{\bbE[U_{n,i}^4]} \cdot \sqrt{\bbP(U_{n,i}^4 > \epsilon^4)} \\
          &\stackrel{(ii)}{\leq} \sqrt{\bbE[U_{n,i}^4]} \cdot \sqrt{\frac{\bbE[U_{n,i}^4]}{\epsilon^4}} \\
          &= \frac{\bbE\sqbrac{U_{n,i}^4}}{\epsilon^2}
        \end{aligned}
      \end{equation*}
      $(i)$ is a consequence of Holder's inequality, and $(ii)$ is obtained through Markov's inequality applied to the final term. Now, if $U_{n,i}$ has kurtosis upper bounded by $\kappa_{\max}$, we would have
      \begin{equation*}
        \begin{aligned}
          \bbE[U_{n,i}^2 \bbI(\abs{U_{n,i}} > \epsilon)] &\leq \frac{\kappa_{\max} \paren{\bbE[U_n,i^2]}^2}{\epsilon^2} \\
          &= \frac{\kappa_{\max}}{\epsilon^2} \cdot \Var(U_{n,i})^2 \\
          &= \frac{\kappa_{\max}}{\epsilon^2} \cdot \paren{\frac{\Var(\Lambda_i)}{\Vbar_n}}^2 \\
          &= \frac{\kappa_{\max}}{\epsilon^2} \cdot \paren{\frac{\Var(\Lambda_i)}{\sum_{j=1}^{n} \Var(\Lambda_j)}}^2 \\
        \end{aligned}
      \end{equation*}

      The sum of all these expectations is thus upper bounded by
      \begin{equation*}
        \begin{aligned}
          \sum_{i = 1}^{n} \bbE\sqbrac{U_{ni}^2 \bbI(\abs{U_{ni} > \epsilon})} &\leq \frac{\kappa_{\max}}{\epsilon^2} \frac{\sum_{i=1}^{n} \Var(\Lambda_i)^2}{\paren{\sum_{j=1}^{n} \Var(\Lambda_j)}^2} \\
        \end{aligned}
      \end{equation*}

      We now apply Lemma \ref{lem:single_term_total_var_limit} followed by Lemma \ref{lem:total_vars_sums_ratio} to get
      \begin{align*}
        \frac{\sum_{i=1}^{n} \Var(\Lambda_i)^2}{\paren{\sum_{j=1}^{n} \Var(\Lambda_j)}^2} \to 0,
      \end{align*}
      which proves the Lindeberg condition
    \item \textit{Convergence of quadratic variation}: $\sum_{i=1}^{n} \Var(U_{n,i} \mid \calF_{i-1}) \to \eta^2$, where $\eta^2$ is a bounded random variable with $\bbP(\eta^2 > 0) = 1$.

      From Lemma \ref{lem:var_sum_ratio_limit}, we infer that
      \begin{align*}
        \sum_{i=1}^{n} \Var(U_{n,i} \mid \calF_{i-1}) = \frac{1}{\Vbar_n} \sum_{i=1}^{n} \Var(\Lambda_i \mid \calF_{i-1}) = \frac{V_n}{\Vbar_n} \stackrel{\text{p}}{\to} \frac{v^*}{\vbar^*},
      \end{align*}
      where $v^* / \vbar^*$ is a strictly positive bounded random variable from our assumption that $\Var(\ell(z)) \in [\nu^2, b^2]$ for all $\ell \in \calG$.
  \end{enumerate}
  Since the two conditions hold, the martingale CLT follows.
\end{proof}

\subsubsection{Proofs of supporting lemmas}
\begin{proof}[Proof of Lemma \ref{lem:single_term_total_var_limit}]
  We observe that
  \begin{align*}
    \Var(\Lambda_i) &= \Var\paren{\frac{\Delta_i(x_i, y_i)}{\Gamma_i} - \frac{\Delta_{i-1}(x_i, y_i)}{\Gamma_{i-1}}} \\
    &= \Var\paren{\frac{\Delta_i(x_i, y_i)}{\Gamma_i} - \frac{(1 - \gamma_i) \Delta_{i-1}(x_i, y_i)}{\Gamma_i}} \\
    &= \Var\paren{\frac{\gamma_i \Delta_{i-1}(x_i, y_i)}{\Gamma_i} + \frac{\Delta_i(x_i, y_i) - \Delta_{i-1}(x_i, y_i)}{\Gamma_i}} \\
    &= \frac{\gamma_i^2}{\Gamma_i^2} \Var\paren{\underbrace{\Delta_{i-1}(x_i, y_i)}_{\text{Term 1}} + \underbrace{\frac{\Delta_i(x_i, y_i) - \Delta_{i-1}(x_i, y_i)}{\Gamma_i}}_{\text{Term 2}}}
  \end{align*}
  We calculate the variances of the two terms separately. 
  \begin{enumerate}
    \item Term 1: We have
      \begin{align*}
        \Var(\Delta_{i-1}(x_i, y_i)) &= \vbar_{i-1}.
      \end{align*}
      From our assumption of $\min_{f \in \calG} \Var(f(x, y)) \geq \nu^2$, we have $\vbar_{i-1} \geq \nu^2$.
    \item Term 2: We have
      \begin{align*}
        \Var(\Delta_i(x, y) - \Delta_{i-1}(x, y)) &= \Var(\Delta_i(x, y) - \Delta_{i-1}(x, y)) \\
        &= \bbE\sqbrac{\paren{(\ell_i(x, y) - \ell_{i-1}(x, y)) - (\bbE[\ell_i(x, y) \mid \calF_{i-1}] - \bbE[\ell_{i-1}(x, y) \mid \calF_{i-1}])}^2} \\
        &= \bbE\sqbrac{(\ell_i(x, y) - \ell_{i-1}(x, y))^2} + \bbE\sqbrac{(\bbE[\ell_i(x, y) \mid \calF_{i-1}] - \bbE[\ell_{i-1}(x, y) \mid \calF_{i-1}])^2} \\
        &\quad -2\bbE\sqbrac{(\ell_i(x, y) - \ell_{i-1}(x, y))(\bbE[\ell_i(x, y) \mid \calF_{i-1}] - \bbE[\ell_{i-1}(x, y) \mid \calF_{i-1}])} \\
      \end{align*}
      Simplifying the third term, we get
      \begin{align*}
        &\bbE\sqbrac{(\ell_i(x, y) - \ell_{i-1}(x, y))(\bbE[\ell_i(x, y) \mid \calF_{i-1}] - \bbE[\ell_{i-1}(x, y) \mid \calF_{i-1}])} \\
        &\quad= \bbE\sqbrac{\bbE\sqbrac{(\ell_i(x, y) - \ell_{i-1}(x, y))(\bbE[\ell_i(x, y) \mid \calF_{i-1}] - \bbE[\ell_{i-1}(x, y) \mid \calF_{i-1}])}\mid \calF_{i-1}} \\
        &\quad= \bbE\sqbrac{(\bbE[\ell_i(x, y) \mid \calF_{i-1}] - \bbE[\ell_{i-1}(x, y) \mid \calF_{i-1}]) \cdot \bbE\sqbrac{(\ell_i(x, y) - \ell_{i-1}(x, y))}\mid \calF_{i-1}} \\
        &\quad= \bbE\sqbrac{(\bbE[\ell_i(x, y) \mid \calF_{i-1}] - \bbE[\ell_{i-1}(x, y) \mid \calF_{i-1}])^2} \\
      \end{align*}
      Substituting this term back in the variance expression, we get
      \begin{align*}
        \Var(\Delta_i(x, y) - \Delta_{i-1}(x, y)) &= \bbE\sqbrac{(\ell_i(x, y) - \ell_{i-1}(x, y))^2} - \bbE\sqbrac{(\bbE[\ell_i(x, y) \mid \calF_{i-1}] - \bbE[\ell_{i-1}(x, y) \mid \calF_{i-1}])^2} \\
        &\leq \bbE\sqbrac{(\ell_i(x, y) - \ell_{i-1}(x, y))^2} \stackrel{(i)}{\leq} \sigma_i^2,
      \end{align*}
      where $(i)$ comes from our uniform stability assumption on the learning algorithm.

      The second term's variance is thus upper bounded as
      \begin{align*}
        \Var\paren{\frac{\Delta_i(x, y) - \Delta_{i-1}(x, y)}{\gamma_i}} &\leq \frac{\sigma_i^2}{\gamma_i^2}.
      \end{align*}
      From our choice of $\gamma_i$ such that $\gamma_i = \omega(\sigma_i)$, we have that $\Var\paren{\frac{\Delta_i(x, y) - \Delta_{i-1}(x, y)}{\gamma_i}} \to 0$.
  \end{enumerate}

  We thus infer that the variance is completely dominated by the first term as $i$ increases. Since we assume that $\vbar_i \to \vbar^* > 0$, we thus have
  \begin{align}
    \label{eq:single_term_var_limit}
    \Var\paren{\Delta_{i-1}(x_i, y_i) + \frac{\Delta_i(x_i, y_i) - \Delta_{i-1}(x_i, y_i)}{\Gamma_i}} \to \vbar^* > 0.
  \end{align}

  \begin{align*}
    \frac{1}{\gamma_i} - \frac{1}{\gamma_{i-1}} &\stackrel{(i)}{\leq} 1 \\
    \gamma_{i-1} - \gamma_i &\leq \gamma_{i-1} \gamma_i \\
    \gamma_{i-1}(1 - \gamma_i) &\leq \gamma_i \\
    1 &\leq \frac{\gamma_i}{\gamma_{i-1}(1 - \gamma_i)} \\
    \gamma_{i-1} &\leq \frac{\gamma_i}{(1 - \gamma_i)} \\
    \frac{\gamma_{i-1}}{\Gamma_{i-1}} &\leq \frac{\gamma_i}{\Gamma_i},
  \end{align*}
  where $(i)$ comes from our assumption on $\set{\gamma_i}$. Thus, $\gamma_i / \Gamma_i$ is a non-increasing sequence. Coupled with \ref{eq:single_term_var_limit}, we can conclude that $\lim\inf_i \Var(\Lambda_i) \geq \gamma_1 \vbar^* > 0$.

  Now, let us consider $\Var(\Lambda_i) / \Var(\Lambda_{i-1})$. We have
  \begin{align*}
    \frac{\Var(\Lambda_i)}{\Var(\Lambda_{i-1})} &= \frac{\gamma_i^2}{\Gamma_{i}^2} \cdot \frac{\Gamma_{i-1}^2}{\gamma_{i-1}^2} \cdot \frac{\Var\paren{\Delta_{i-1}(x_i, y_i) + \frac{\Delta_i(x_i, y_i) - \Delta_{i-1}(x_i, y_i)}{\Gamma_i}}}{\Var\paren{\Delta_{i-2}(x_{i-1}, y_{i-1}) + \frac{\Delta_{i-1}(x_{i-1}, y_{i-1}) - \Delta_{i-2}(x_{i-1}, y_{i-1})}{\Gamma_{i-1}}}} \\
    &= \paren{\frac{\gamma_i}{\gamma_{i-1}(1 - \gamma_i)}}^2 \cdot  \frac{\Var\paren{\Delta_{i-1}(x_i, y_i) + \frac{\Delta_i(x_i, y_i) - \Delta_{i-1}(x_i, y_i)}{\Gamma_i}}}{\Var\paren{\Delta_{i-2}(x_{i-1}, y_{i-1}) + \frac{\Delta_{i-1}(x_{i-1}, y_{i-1}) - \Delta_{i-2}(x_{i-1}, y_{i-1})}{\Gamma_{i-1}}}} \\
  \end{align*}
  From \ref{eq:single_term_var_limit}, we know that the second term goes to one. For the first term, we know from proposition [PROPOSITION] that $\gamma_i / \gamma_{i-1} \to 1$ and $(1 - \gamma_i) \to 1$, and thus, $\gamma_i / (\gamma_{i-1}(1 - \gamma_i)) \to 1$. Thus, we conclude that $\Var(\Lambda_i) / \Var(\Lambda_{i-1}) \to 1$.
\end{proof}

\begin{proof}[Proof of Lemma \ref{lem:total_vars_sums_ratio}]
We have
  \begin{equation*}
    \begin{aligned}
      R_n = \frac{\sum_{i=1}^{n} A_i^2}{\paren{\sum_{i=1}^{n} A_i}^2} &\leq \frac{\paren{\sum_{i=1}^{n} A_i} \cdot \max_i \in [n] A_i}{\paren{\sum_{i=1}^{n} A_i}^2} \\
      &= \frac{\max_{i \in [n]} A_i}{\sum_{i=1}^{n} A_i} \\
    \end{aligned}
  \end{equation*}
  We now consider two cases based on how many times $A_{\max, n} = \max_{i \in [n]} A_i$ changes with $n$.

  \begin{enumerate}
    \item \textit{Case 1: $A_{\max, n}$ changes finitely many times}: In this case, there exists some $n_0$ such that for all $n \geq n_0$, $A_{\max, n} = A_{\max, n_0}$. Since $\lim\inf_i A_i \geq m$ for all $i$, we would have $\sum_{i=1}^{n} A_i \to \infty$. Thus, $A_{\max, n} / \sum_{i=1}^{n} A_i \to 0$.

    \item \textit{Case 2: $A_{\max, n}$ changes infinitely often}: Let us consider the quantity $A_n / (\sum_{i=1}^{n} A_i)$, i.e., the ratio of the last term to the sum of all the terms. We have two situations depending on whether $A_n$ is the maximum term:
    \begin{enumerate}[label = (\roman*)]
      \item If $A_n$ is the maximum term, then $A_{\max, n} / (\sum_{i=1}^{n} A_i) = A_n / (\sum_{i=1}^{n} A_i)$.
      \item If $A_n$ is not the maximum term, then $A_{\max, n} = A_{\max, n-1}$, and $A_{\max, n} / (\sum_{i=1}^{n} A_i) \leq A_{\max, n} / (\sum_{i=1}^{n-1} A_i)$.
    \end{enumerate}
    Thus, the sequence $A_{\max, n} / (\sum_{i=1}^{n} A_i)$ can increase only when $A_{\max, n} = A_n$, decreasing otherwise. If we show that $A_n / (\sum_{i=1}^{n} A_i) \to 0$, then we can conclude that $A_{\max, n} / (\sum_{i=1}^{n} A_i) \to 0$ as well.

    To do this, we define $S_n = (\sum_{i=1}^{n} A_i) / A_n$, which is the inverse of the ratio defined above.
    We can express $S_n$ through the following recursion:
    \begin{align*}
      S_n = \begin{cases} 1, \quad& n = 1 \\
      1 + \frac{A_{n-1}}{A_n} S_{n-1}, \quad& n > 1 \end{cases}
    \end{align*}
    Consider some arbitrary $\epsilon > 0$. We know that $A_{n-1}/A_n \leq 1$, with $A_{n-1}/A_n \to 1$. Thus, there exists some $n_0(\epsilon)$ such that for all $n \geq n_0(\epsilon)$, $A_{n-1}/A_n \geq 1 - \epsilon$. Let $S_n'$ be another sequence defined for $n \geq n_0(\epsilon)$ as
    \begin{align*}
      S_n' = \begin{cases} 1, \quad& n = n_0(\epsilon) \\
      1 + (1 - \epsilon) S_{n-1}', \quad& n > n_0(\epsilon) \end{cases}
    \end{align*}
    Clearly, $S_n' \leq S_n$ for all $n \geq n_0(\epsilon)$. We also note that $\lim_{n \to \infty} S_n' = \frac{1}{\epsilon}$. Thus, we have
    \begin{align*}
      \frac{1}{\epsilon} = \lim_{n \to \infty} S_n' \leq \lim_{n \to \infty} S_n.
    \end{align*}
    Since $\epsilon$ is arbitrary, we have $\lim_{n \to \infty} S_n = \infty$. Thus, $A_n / \sum_{i=1}^{n} A_i \to 0$.
  \end{enumerate}
  Since we have shown that $A_{\max, n} / (\sum_{i=1}^{n} A_i \to 0)$ for both cases, we can conclude that $R_n \to 0$.
\end{proof}

\begin{proof}[Proof of Lemma \ref{lem:var_sum_ratio_limit}]
  Let $v_i = \Var(\Lambda_i \mid \calF_{i-1})$ and $\vbar_i = \Var(\Lambda_i)$. We then have $V_n = \sum_{i=1}^{n} v_i$ and $\Vbar_n = \sum_{i=1}^{n} \vbar_i$. We also have
  \begin{align*}
    v_i &= \Var(\Lambda_i \mid \calF_{i-1}) \\
    &= \Var\paren{\frac{\Delta_i(x_i, y_i)}{\Gamma_i} - \frac{\Delta_{i-1}(x_i, y_i)}{\Gamma_{i-1}} \mid \calF_{i-1}} \\
    &\stackrel{(i)}{=} \frac{\gamma_i^2}{\Gamma_i^2} \Var\paren{\Delta_i(x_i, y_i) + \frac{(\Delta_i(x_i, y_i) - \Delta_{i-1}(x_i, y_i))}{\gamma_i} \mid \calF_{i-1}},
  \end{align*}
  where $(i)$ comes from a decomposition similar to that in Lemma \ref{lem:single_term_total_var_limit}. For notational convenience we use 
  \begin{align*}
    u_i = \Var\paren{\Delta_i(x_i, y_i) + \frac{(\Delta_i(x_i, y_i) - \Delta_{i-1}(x_i, y_i))}{\gamma_i} \mid \calF_{i-1}}
  \end{align*}

  From our assumption that $\gamma_i = \omega(\sigma_i)$, we know that $\sigma_i / \gamma_i \to 0$. Moreover, as $\Var(\ell(x, y) \in [\nu^2, b^2]$ for all $\ell \in \calG$, we know that $\Var(\Delta_i(x_i, y_i)) \geq \nu^2$ for all $i$. Thus, there should be some $i_0$ such that for all $i \geq i_0$, $\sigma_i / \gamma_i < \nu^2$, where we would have
  We then have
  \begin{align*}
    \paren{\sqrt{\Var(\Delta_i(x_i, y_i) \mid \calF_{t-1})} - \frac{\sigma_i}{\gamma_i}}^2 \leq u_i \leq \paren{\sqrt{\Var(\Delta_i(x_i, y_i) \mid \calF_{t-1})} + \frac{\sigma_i}{\gamma_i}}^2
  \end{align*}
  As $\sigma_i / \gamma_i \to 0$, we would thus have $u_i \stackrel{\text{p}}{\to} v^*$.

  We also note that $u_i$ is bounded from above by $(b + \sup_i \sigma_i / \gamma_i)$, where the supremum is finite since $\sigma_i, \gamma_i > 0$ and $\sigma_i / \gamma_i \to 0$. Thus, we would have $u_i \stackrel{L_1}{\to} v^*$, which implies that $\bbE[\abs{u_i - v^*}] \to 0$.

  Let us define two sequences:
  \begin{align*}
    U_n &= \sum_{i=1}^{n} \frac{\gamma_i^2}{\Gamma_i^2} \bbE[\abs{u_i -v^*}], \quad \text{and} \quad W_n = \sum_{i=1}^{n} \frac{\gamma_i^2}{\Gamma_i^2}
  \end{align*}
  We note that $W_n$ is a monotonically increasing sequence, and from our observation that $\gamma_i^2 / \Gamma_i^2$ is an increasing sequence in Lemma \ref{lem:single_term_total_var_limit}, we have $W_n \to \infty$. We then have, for $n > 1$,
  \begin{align*}
    \frac{U_n - U_{n-1}}{W_n - W_{n-1}} &= \bbE[\abs{u_i - v^*}] \to 0
  \end{align*}
  Thus, by the Stolz-Cesaro theorem, we have $U_n / W_n \to 0$. By repeated application of the triangle inequality, we have
  \begin{align*}
    \bbE\sqbrac{\abs{\frac{V_n}{W_n} - v^*}} &= \bbE\sqbrac{\abs{\frac{\sum_{i=1}^{n} \gamma_i^2 u_i / \Gamma_i^2}{W_n} - v^*}} \\
    &\leq \frac{1}{W_n} \cdot \sum_{i=1}^{n} \frac{\gamma_i^2}{\Gamma_i^2} \bbE[\abs{u_i - v^*}] = \frac{U_n}{W_n} \to 0
  \end{align*}
  Thus, we have $V_n / W_n \stackrel{L_1}{\to} v^*$, which in turn implies convergence in probability.

  By a similar argument but for total variance, we obtain $\Vbar_n / W_n \to \vbar^* = \bbE[v^*]$ (since $\vbar_n$ is a deterministic quantity, $\Vbar_n / W_n$ is also deterministic). By applying Slutsky's theorem, we thus get
  \begin{align*}
    \frac{V_n}{\Vbar_n} = \frac{V_n / W_n}{\Vbar_n / W_n} \to \frac{v^*}{\vbar^*}
  \end{align*}
\end{proof}

\subsection{Proof of Theorem \ref{thm:oeuvre_concentration}}
\label{apdx:oeuvre_concentration_proof}

\paragraph{Fixed-time concentration bound}
\begin{proof}
  We consider the martingale $M_t / \Gamma_t$ defined on the filtration $\set{\calF_{t}}$. The corresponding martingale difference sequence is given by
  \begin{align*}
    U_t &= \frac{M_t}{\Gamma_t} - \frac{M_{t-1}}{\Gamma_{t-1}} \\
    &= \frac{\Delta_t(z_t) - (1 - \gamma_t) \Delta_{t-1}(z_t)}{\Gamma_t} \\
    &= \frac{\gamma_t \Delta_t(z_t)}{\Gamma_t} + \frac{(\Delta_t(z_t) - \Delta_{t-1}(z_t))}{\Gamma_{t-1}}
  \end{align*}
  We note that the first term is bounded between an interval of size $[-b \gamma_t / \Gamma_t, b \gamma_t / \Gamma_t]$. From our assumption of uniform stability, the second term is bounded within an interval of size $[-\sigma_t / \Gamma_{t-1}, \sigma_t / \Gamma_{t-1}]$. Thus, the sum of the two terms is bounded between $[-(b \gamma_t + (1 - \gamma_t) \sigma_t) / \Gamma_t, (b \gamma_t + (1 - \gamma_t) \sigma_t) / \Gamma_t]$. From the Azuma-Hoeffding inequality \cite{wainwright2019high}, we thus have
  \begin{align*}
    \bbP\paren{\abs{\frac{M_T}{\Gamma_T}} \geq \frac{\epsilon}{\Gamma_T}} &\leq 2 \exp\paren{-\frac{2 \epsilon^2 / \Gamma_T^2}{4 \sum_{t=1}^{T} (b \gamma_t + (1 - \gamma_t) \sigma_t)^2 / \Gamma_t^2}} \\
    &= 2 \exp\paren{-\frac{\epsilon^2 / \Gamma_T^2}{2 \Vup_T / \Gamma_T^2}} \\
    \implies \bbP\paren{\abs{M_T} \geq \epsilon} &\leq 2 \exp\paren{-\frac{\epsilon^2 }{2 \Vup_T}}
  \end{align*}
\end{proof}

\paragraph{Time-uniform concentration bound}
\begin{proof}
  From the master theorem from \cite{howard2020time}, we know that for a martingale with sub-Gaussian increments $U_t$, we have
  \begin{align*}
    \bbP\paren{\exists t \in [T]: \abs{\sum_{i=1}^{T} U_T} \geq \frac{x}{2} + \frac{x}{m} V_t} &\leq 2\exp\paren{-\frac{x^2}{2m}},
  \end{align*}
  where $U_t = M_t / \Gamma_t - M_{t-1} / \Gamma_{t-1}$, and $V_T = \sum_{t=1}^{T} \Var(U_t \mid \calF_{t-1})$. For our martingale $M_t / \Gamma_t$, we know that $\Vup_t \Gamma_t^2 \geq V_t$ for all $t$, which implies that
  \begin{align*}
    \bbP\paren{\exists t \in [T]: \abs{\sum_{i=1}^{T} U_T} \geq \frac{x}{2} + \frac{x}{m} \frac{\Vup_t}{\Gamma_t^2}} &\leq \bbP\paren{\exists t \in [T]: \abs{\sum_{i=1}^{T} U_T} \geq \frac{x}{2} + \frac{x}{m} V_t} \\
    &\leq 2\exp\paren{-\frac{x^2}{2m}} \\
    \implies \bbP\paren{\exists t \in [T]: \abs{\frac{M_t}{\Gamma_t}} \geq \frac{x}{2} + \frac{x}{m} \frac{\Vup_t}{\Gamma_t^2}} &\leq 2\exp\paren{-\frac{x^2}{2m}} \\
    \bbP\paren{\exists t \in [T]: \abs{M_t} \geq \frac{x}{2} \Gamma_t + \frac{x}{m} \Vup_t \Gamma_t} &\leq 2\exp\paren{-\frac{x^2}{2m}} \\
  \end{align*}
  Let $x^2/2m = c$, which means that $m = x^2 / 2c$. The bound inside the probability at time $T$ is
  \begin{align*}
    \frac{x}{2} \Gamma_T + \frac{x}{m} \Vup_T \Gamma_T &= \frac{x}{2} \Gamma_T + \frac{2c}{x} \cdot \frac{\Vup_T}{\Gamma_T}.
  \end{align*}
  Minimizing this quantity w.r.t. $x$, we get
  \begin{align*}
    x = \frac{\sqrt{2c \Vup_T}}{\Gamma_T}, \quad \text{and} \quad m = \frac{c \Vup_T}{\Gamma_T^2}
  \end{align*}
  Thus, our bound becomes
  \begin{align*}
      \bbP\paren{\exists t \in [T]: \abs{M_t} \geq \sqrt{\frac{c}{2 V_T}} \cdot \frac{(V_T \Gamma_t^2 + V_t \Gamma_t^2)}{\Gamma_T \Gamma_t}} &\leq 2\exp\paren{-c}
  \end{align*}

\end{proof}

\subsection{Proof of Lemma \ref{lem:adaptive_const_est}}
\label{apdx:adaptive_const_est_proof}

\begin{proof}
  Let
  \begin{align*}
    V_{1} &= V_{1}' = b^{2} \\
    V_{t} &= (\gamma_{t} b + (1 - \gamma_{t}) \gamma_{t})^{2} + V_{t-1} \\
    V_{t}' &= (\gamma_{t} b + (1 - \gamma_{t}) c \gamma_{t})^{2} + V_{t-1}' \\
  \end{align*}

  We consider two cases:
  \begin{enumerate}
          \item $c \geq 1$: Here, we have
          \begin{align*}
            V_{1} &= b^{2} = V_{1}'
            V_{t} - V_{t-1} &= (\gamma_{t} b + (1 - \gamma_{t}) \sigma_{t})^{2} \leq (\gamma_{t} b + (1 - \gamma_{t}) c \sigma_{t})^{2} = V_{t}' - V_{t-1}'
          \end{align*}
          Adding the inequalities up for $t \in \set{1, \ldots, T}$, we get $V_{T} \leq V_{T}'$.
          \item $c < 1$: Here, we have
          \begin{align*}
            V_{1} &= b^{2} < b^{2} / c^{2} = V_{1}' / c^{2} \\
            V_{t} - V_{t-1} &= (\gamma_{t} b + (1 - \gamma_{t}) \sigma_{t})^{2} = \frac{(\gamma_{t} bc + (1 - \gamma_{t}) c \sigma_{t})^{2}}{c^{2}} \\
            &\leq \frac{(\gamma_{t} b + (1 - \gamma_{t}) c \sigma_{t})^{2}}{c^{2}} = V_{t}' - V_{t-1}'
          \end{align*}
          Adding the inequalities up for $t \in \set{1, \ldots, T}$, we get $V_{T} \leq V_{T}' / c^{2}$.
  \end{enumerate}
\end{proof}

\section{Implementation Details}
\label{apdx:implementation_details}

\paragraph{Metrics.} Let $\overline{\ell}_t = \bbE[\ell_t(Z) \mid \calF_{t-1}]$ be the true expected loss of the model at time $t$, and $\hat{\ell}_t$ be the estiamted loss. We consider the following metrics to compare the performance of different estimators:
\begin{itemize}
  \item \textit{Root mean squared error (MSE)}: $\sqrt{\frac{1}{T} \sum_{t=1}^{T} (\hat{\ell}_t - \overline{\ell_t})^2}$.
  \item \textit{Mean absolute error (MAE)}: $\frac{1}{T} \sum_{t=1}^{T} \abs{\hat{\ell}_t - \overline{\ell_t}}$.
  \item \textit{Bias}: $\frac{1}{T} \sum_{t=1}^{T} (\hat{\ell}_t - \overline{\ell_t})$.
\end{itemize}

\paragraph{Baseline estimators.} We consider the following hyperparameter sweeps for the baseline estimators:
\begin{itemize}
  \item \textit{Sliding window (SW)}: $\text{Window size } \in \set{10, 50, 50, 100, 200, 400, 600, 800, 1000}$.
  \item \textit{Exponential moving average (EMA)}: $\text{Decay factor } \in \set{0.1, 0.05, 0.01, 0.005, 0.001}$.
  \item \textit{Fading factor prequential (FFPreq)}: $\text{Decay factor } \in \set{0.8, 0.9, 0.95, 0.99, 0.999, 0.9999, 0.99999}$.
  \item \textit{ADWIN}: $\text{Sensitivity } \delta \in \set{10^{-2}, 10^{-3}, 10^{-4}, 10^{-5}, 10^{-6}} \cup \set{5 \cdot 10^{-2}, 5 \cdot 10^{-3}, 5 \cdot 10^{-4}, 5 \cdot 10^{-5}} \cup \set{0.1, 0.2, \ldots, 0.9}$.
\end{itemize}

At time step $t$, these estimators are given the prequential evaluation $\ell_t(z_t)$ for update.

\subsection{Linear Regression}
\paragraph{Data.} We generate synthetic data for $d$-dimensional linear regression by first sampling the true weight $\bfw^*$ uniformly from the $d$-dimensional unit hypersphere. We then sample $\bfx_t \sim \calN(0, \Sigma)$, where $\Sigma \in \bbR^{d \times d}$ is a covariance matrix such that $\Sigma_{ii} = 1$ for all $i \in [d]$. We then set $y_t = (\bfw^*)^T \bfx_t + \epsilon$, where $\epsilon \sim \calN(0, \sigma^2)$ We vary the standard deviation of the noise, $\sigma \in \set{0.005, 0.05, 0.5}$.

\paragraph{Loss.} We consider the scaled mean squared error $\ell_t(\hat{y}, y) = (y - \hat{y})^2 / d$ as the loss function. The division by $d$ ensures that we can use the same learning rate across all dimensions and noise levels. We set the true loss at time $t$ as $\overline{\ell}_t = (\bfw_t - \bfw^*)^T \Sigma (\bfw_t - \bfw^*) + \sigma^2$, where $\bfw_t$ is the weight vector of the model at time $t$. All estimators try to estimate this true loss with access to $\ell_t(z_t) = \ell(\bfw_t^T \bfx_t, y_t)$. OEUVRE additionally has access to $\ell_{t-1}(z_t) = \ell(\bfw_{t-1}^T \bfx_t, y_t)$.

\paragraph{Learning algorithm.} We use online gradient descent with a learning rate of $\eta_t = \eta_0 / \sqrt{t}$, with $\eta_0 = 0.01$. We initialize the weight vector $\bfw_0$ to be the zero vector.

\subsection{Prediction with Expert Advice}
\paragraph{Data.} We generate synthetic data for prediction with expert advice with $d$ experts by sampling each expert $i$'s loss at time $t$ from distribution $\calD_i$ in an iid manner. The distributions $\calD_i$ are independent of each other. We use the following choices of distributions:
\begin{itemize}
  \item \textit{Beta distribution}: $\calD_i = \text{Beta}(\alpha_i, \beta_i)$, where $\alpha_i, \beta_i \in \set{1, \cdots, 9}$.
  \item \textit{Bernoulli distribution}: $\calD_i = \text{Bernoulli}(p_i)$, where $p_i \sim \calU(0.01, 0.99)$ are sampled uniformly.
\end{itemize}

\paragraph{Loss.} Let $\overline{\bfell}$ be the vector of expected losses of the experts, and $\bfp_t$ be the distribution over the experts learned by the algorithm at time $t$. The expected loss is given by $\overline{\bfell}^T \bfp_t$. The loss observed at time $t$ is given by $\ell_t(z_t) = \bfell_t^T \bfp_t$, where $\bfell_t$ is the vector of losses of the experts at time $t$. OEUVRE additionally has access to $\ell_{t-1}(z_t) = \bfell_t^T \bfp_{t-1}$.

\paragraph{Learning algorithm.} We use the Hedge algorithm \cite{mourtada2019optimality} with a learning rate of $\eta_t = \sqrt{\log(d) / t}$. We initialize the distribution over experts $\bfp_0$ to be the uniform distribution.

\subsection{Neural Networks}
\paragraph{Data.} We use the MNIST \cite{deng2012mnist} and EMNIST-digits \cite{cohen2017emnist} datasets for our experiments.

\paragraph{Loss.} We use the cross-entropy loss for our experiments. The goal is to estimate the expected loss of the currently learned model on the test split of the data. We obtain a proxy for the expected test loss by evaluating on 1024 test samples at each time step. Each estimator is given 128 independently sampled test samples at each time to update its estimate. OEUVRE additionally has access to the loss of the previous model on the same 128 samples.

\paragraph{Learning algorithm.} We use a simple neural network with one hidden layer consisting of 128 ReLU units. We train the model using mini-batch SGD using the AdamW optimizer with a learning rate of $5 \times 10^{-4}$. Each experiment is repeated for 10 epochs.

\begin{figure}
  \centering
  \includegraphics[width = 0.6\textwidth]{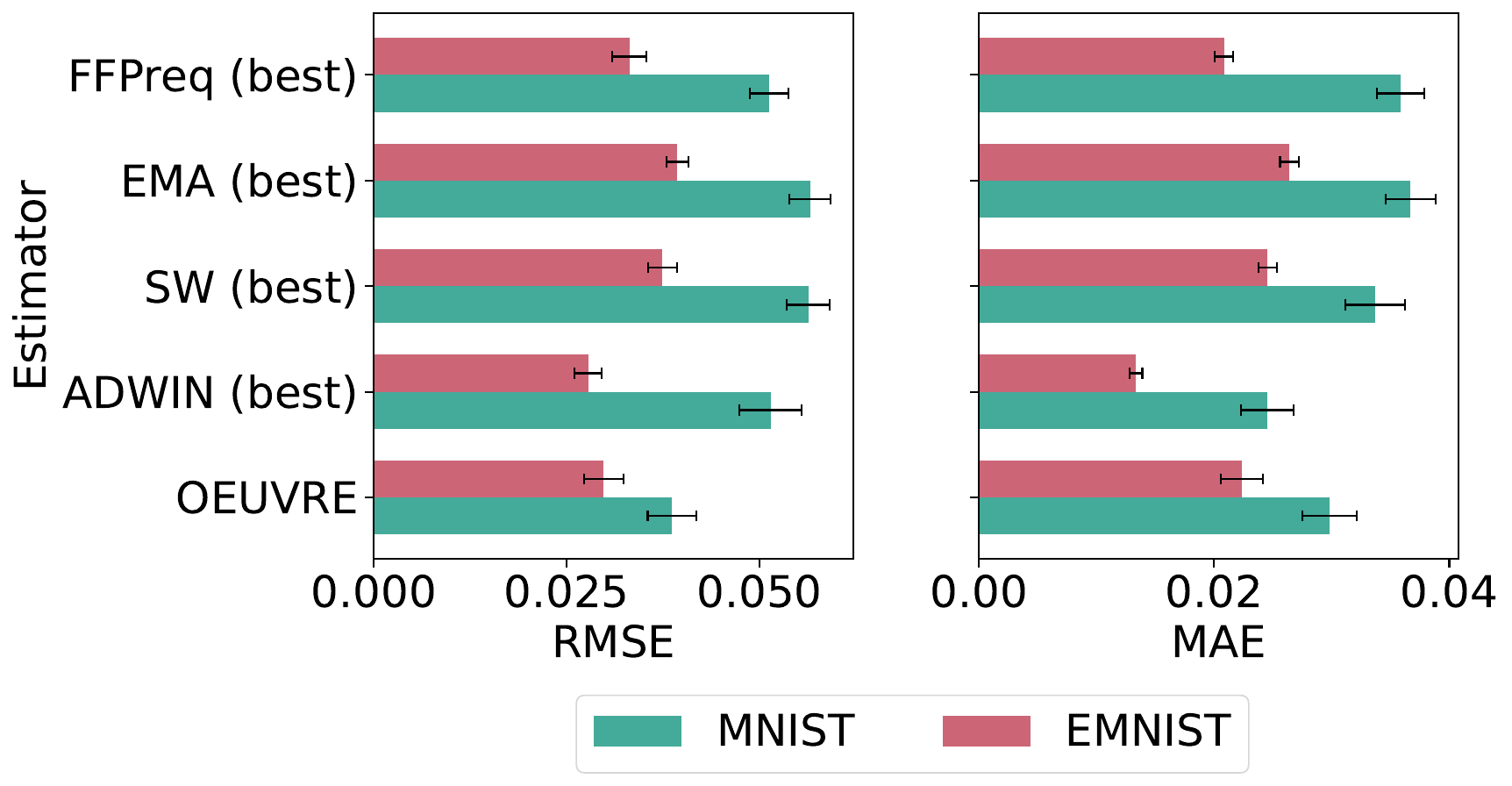}
  \caption{Performance of different loss estimation methods on the MNIST and EMNIST dataset with batch size 8.}
  \label{fig:nn_batchsz_8}
\end{figure}

\subsection{Logistic Regression}
\paragraph{Data.} We use the Diabetes Health Indicators dataset to perform binary classification. We shuffle the data randomly and use the first 10,000 samples to train our model.

\paragraph{Loss.} We use the cross-entropy loss for our experiments. The goal is to estimate the expected loss of the currently learned model on the entire training set.

\paragraph{Learning algorithm.} We use a logistic regression model trained using online gradient with a learning rate of $\eta_t = \eta_0 / \sqrt{t}$, with $\eta_0 = 0.05$ and use Polyak-Ruppert (PR) averaging. PR-averaging results in an algorithm which is $1/t$-uniformly stable (Table \ref{tab:stability_results}). We initialize the weight vector $\bfw_0$ to be the zero vector.

\section{Additional experimental results}
In this section, we provide two additional sets of experimental results. We first show the comparative performance of OEUVRE against the best settings of all baselines for the linear regression and prediction with expert advice tasks in Figures \ref{fig:apdx_linreg_all_estimators} and \ref{fig:apdx_expert_advice_all_estimators}, respectively. We also compare the performance of these estimators on another metric: the time-averaged bias against the ground truth. We present these results for all three tasks. For the linear regression and prediction with expert advice tasks, we also add the simple Prequential estimator, which computes the time-averaged prequential loss at each time step, as a baseline.

In general, we observe that the best-performing abaseline varies across tasks and metrics. OEUVRE, however, consistently matches or outperforms the best-performing baseline, proving to be a good general performer across tasks and metrics.

\paragraph{Linear regression.} We present the results for the linear regression task in Figure \ref{fig:apdx_linreg_all_estimators}. We observe that OEUVRE's performance is comparable to the best settings of SW, EMA, and FFPreq, with ADWIN and the Prequential estimator performing much worse. This is true across all three metrics of RMSE, MAE, and bias.

\begin{figure}[h]
  \centering
  \includegraphics[width = 0.7\textwidth]{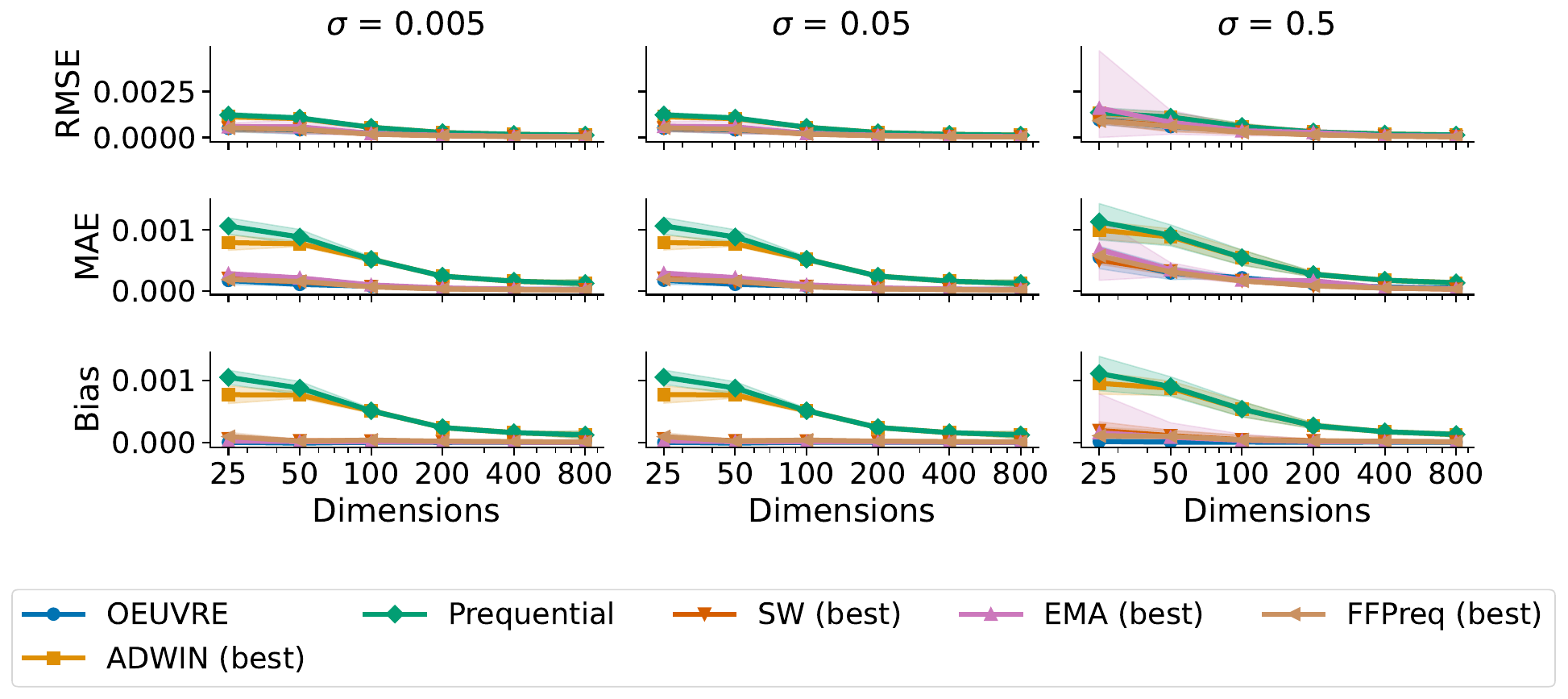}
  % \caption{Performance comparison of different estimators on the online ridge regression task.}
  \caption{Performance comparison of OEUVRE against the best hyperparameter setting for each baseline for the linear regression task. OEUVRE achieves competitive RMSE, MAE, and bias when compared to the best baseline.}
  \label{fig:apdx_linreg_all_estimators}
\end{figure}

\paragraph{Prediction with expert advice.} We present the results for the prediction with expert advice task in Figure \ref{fig:apdx_expert_advice_all_estimators}. EMA performs the worst across both distributions and all metrics. We observe that for the Beta distribution, OEUVRE clearly gives the lowest RMSE and bias across dimensions. For the Bernoulli distribution, OEUVRE's RMSE is comparable to the best, and its bias is the lowest across all dimensions. OEUVRE's MAE is high for lower dimensions ($d = 25, 50$), but it then decreases and becomes comparable to the other estimators for higher dimensions.

\begin{figure}[h]
  \centering
  \includegraphics[width = 0.7\textwidth]{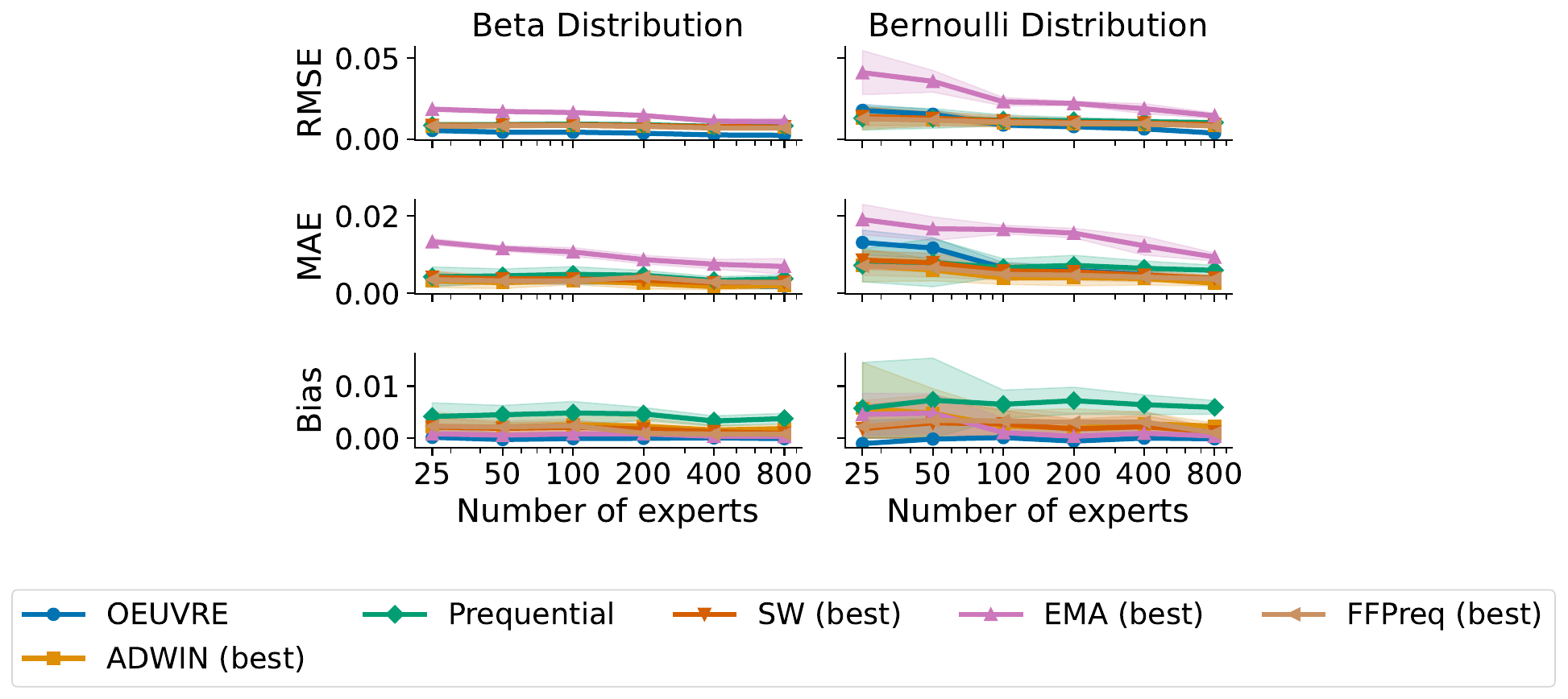}
  % \caption{Performance comparison of different estimators on the online ridge regression task.}
  \caption{Performance comparison of OEUVRE against the best hyperparameter setting for each baseline for the prediction with expert advice task. OEUVRE achieves competitive RMSE, MAE, and bias when compared to the best baseline.}
  \label{fig:apdx_expert_advice_all_estimators}
\end{figure}

\paragraph{Neural networks.} We present the results for the neural network task in Figures \ref{fig:apdx_neural_networks_batchsz_128_bias} and \ref{fig:apdx_neural_networks_batchsz_8_bias} for batch sizes of 128 and 8, respectively. We observe that OEUVRE has the lowest RMSE, MAE and bias for both MNIST and EMNIST datasets for batch size 128. However, the variance of the bias across runs is high. For batch size 8, the performance is more mixed. For the MNIST dataset, ADWIN has lower RMSE and MAE than OEUVRE. For EMNIST, OEUVRE has the lowest RMSE and the second-lowest MAE. OEUVRE's bias is the lowest acros bothd datasets; however, we again observe high variance across runs.

\begin{figure}[h]
  \centering
  \includegraphics[width = 0.7\textwidth]{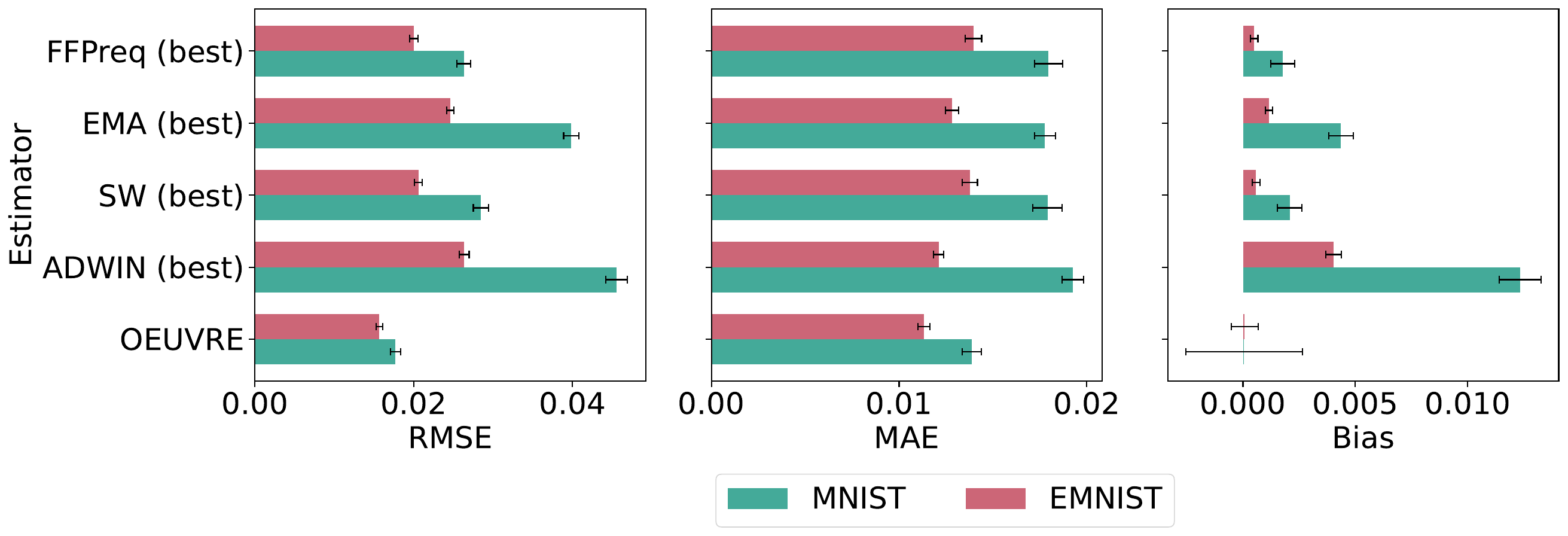}
  % \caption{Performance comparison of different estimators on the online ridge regression task.}
  \caption{Performance comparison of OEUVRE against the best hyperparameter setting for each baseline for the neural network task with 128 samples for estimation at each time step. OEUVRE achieves competitive RMSE, MAE, and bias when compared to the best baseline.}
  \label{fig:apdx_neural_networks_batchsz_128_bias}
\end{figure}

\begin{figure}[h]
  \centering
  \includegraphics[width = 0.7\textwidth]{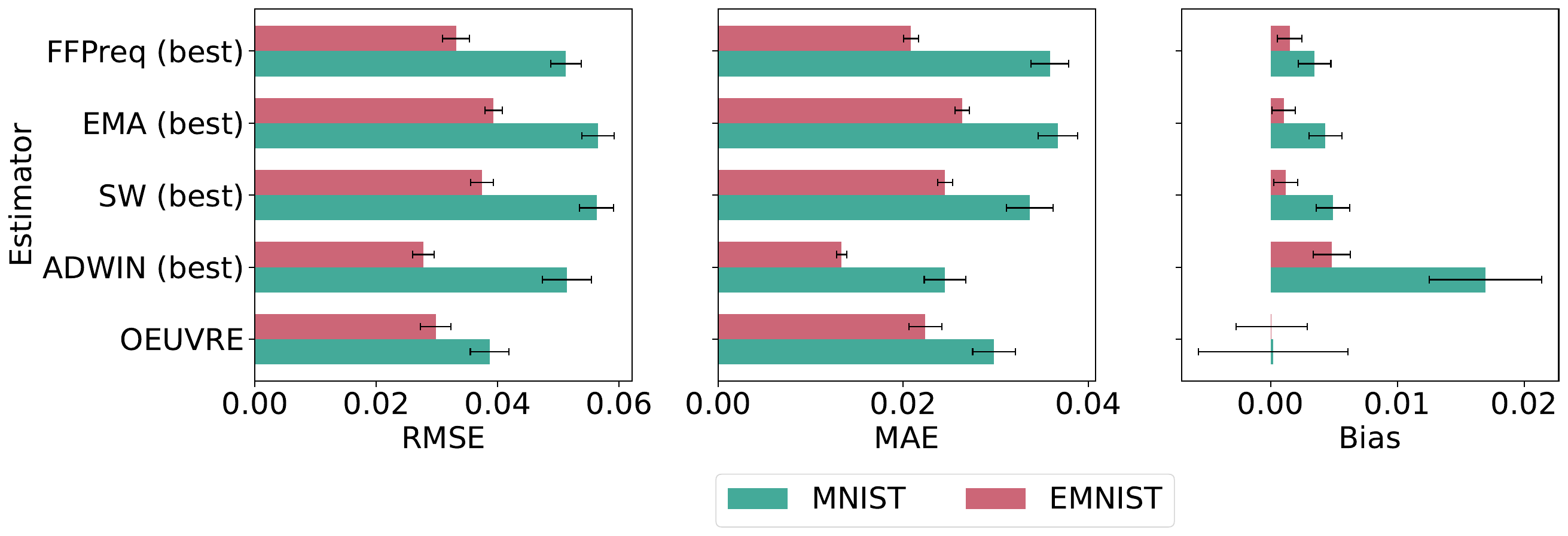}
  % \caption{Performance comparison of different estimators on the online ridge regression task.}
  \caption{Performance comparison of OEUVRE against the best baseline and the median baseline for the neural network task with 8 samples for estimation at each time step. OEUVRE achieves competitive RMSE and bias when compared to the best baseline.}
  \label{fig:apdx_neural_networks_batchsz_8_bias}
\end{figure}

\section{Experiments on non-stationary datasets}
\label{apdx:non_stationary_data_expts}
We conduct further experiments comparing OEUVRE to the baselines on real-world non-stationary datasets. We consider the following datasets: 1) \textit{UCI Bike sharing} \cite{bike_sharing_275}, 2) \textit{Bikes} from the River package \cite{montiel2021river}, 3) \textit{Chick Weights} from the River package \cite{montiel2021river}, and 4) \textit{Wine Quality} from the River package \cite{montiel2021river}. We use a linear regression model trained using online gradient descent to perform regression on these datasets with learning rate $\eta_t = \eta_0 / \sqrt{t}$, with $\eta_0 = 0.01$. The goal is to estimate the expected MSE of the currently learned model at the current time step. The ground truth is estimated by calculating the loss of the model on a lookahead of 50 samples. 

\paragraph{Bike Sharing dataset.} We present our results for this dataset in Figure \ref{fig:perf_comp_bike_sharing_dataset} and Table \ref{tab:perf_comp_bike_sharing_dataset}. We observe that OEUVRE follows the true loss fairly well, achieving performance comparable to the best EMA and SW baselines. This dataset has gradual and seasonal concept drift, which also allows OEUVRE's estimates to change gradually.

\begin{figure}[htbp]
  \centering
  \includegraphics[width = \textwidth]{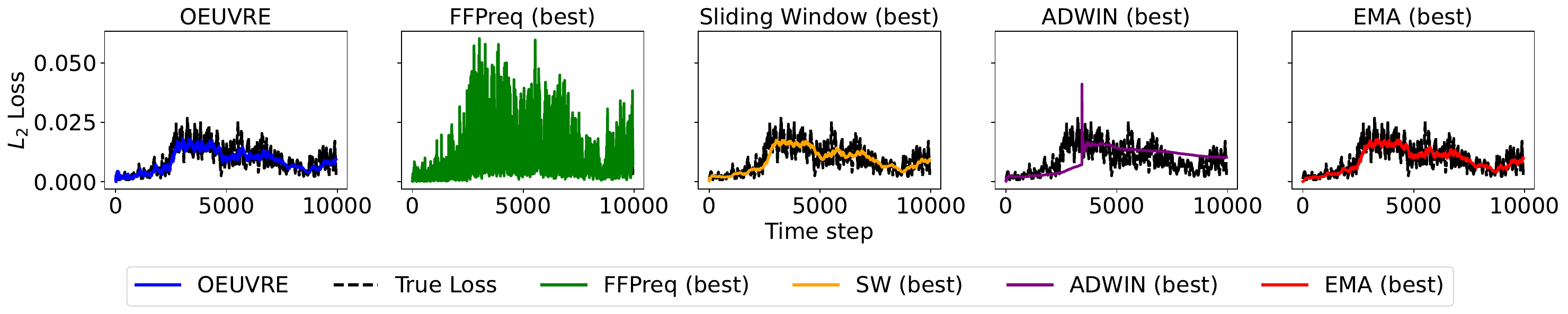}
  \caption{Performance of different algorithms on the Bike Sharing dataset.}
  \label{fig:perf_comp_bike_sharing_dataset}
\end{figure}

\paragraph{Bikes dataset.} We present our results for this dataset in Figure \ref{fig:perf_comp_bikes} and Table \ref{tab:perf_comp_bikes}. We observe that EMA and SW with optimal hyperparameters perform the best, followed by OEUVRE and FFPreq. The Bikes dataset has more extreme changes in the observed loss as compared to the Bike Sharing dataset, which makes it slightly more difficult for OEUVRE to adapt quickly enough.
\begin{figure}[htbp]
  \centering
  \includegraphics[width = \textwidth]{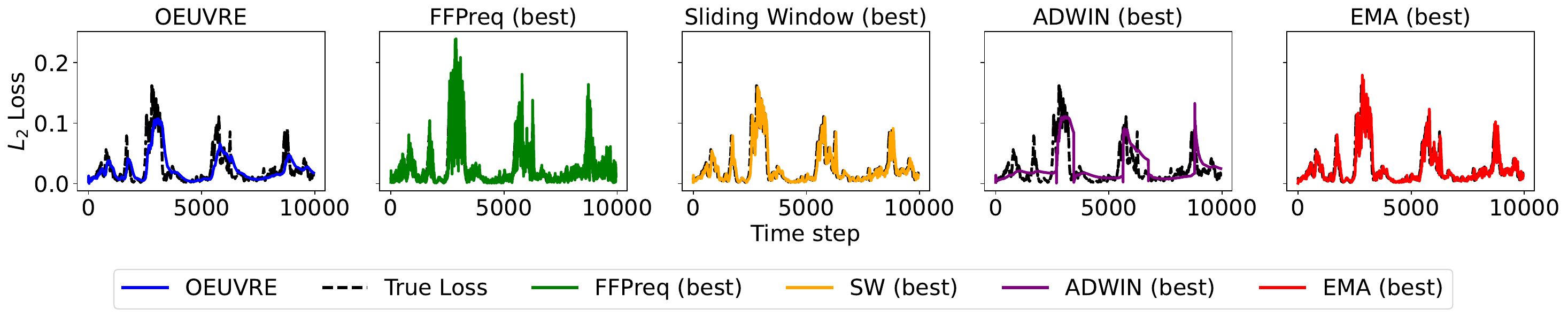}
  \caption{Performance of different algorithms on the Bikes dataset.}
  \label{fig:perf_comp_bikes}
\end{figure}

\paragraph{Chick Weights dataset.} We present our results for this dataset in Figure \ref{fig:perf_comp_chick_weights} and Table \ref{tab:perf_comp_chick_weights}. We observe that the true loss increases rapidly with time, which makes it difficult for all estimators to track the true loss. Although the best settings of EMA, SW, and FFPreq perform better than OEUVRE, OEUVRE still achieves a competitive result when compared with the simple prequential and ADWIN estimators.

\begin{figure}[h]
  \centering
  \includegraphics[width = \textwidth]{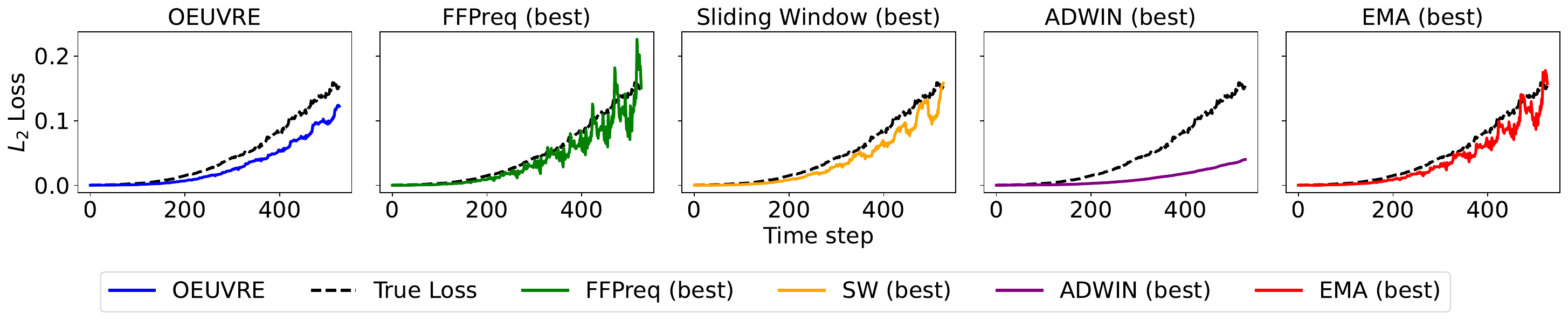}
  \caption{Performance of different algorithms on the Chick Weights dataset.}
  \label{fig:perf_comp_chick_weights}
\end{figure}

\paragraph{Wine Quality dataset.} We present our results for this dataset in Figure \ref{fig:perf_comp_wine_quality_dataset} and Table \ref{tab:perf_comp_wine_quality_dataset}. The ground truth has a lot of variability with time, which makes id difficult for OEUVRE to track it accurately -- we observe in Figure \ref{fig:perf_comp_wine_quality_dataset} that OEUVRE's estimate does not change as rapidly. SW and EMA with optimal hyperparameters perform the best, and OEUVRE's performance is third-best overall.

\begin{figure}[h]
  \centering
  \includegraphics[width = \textwidth]{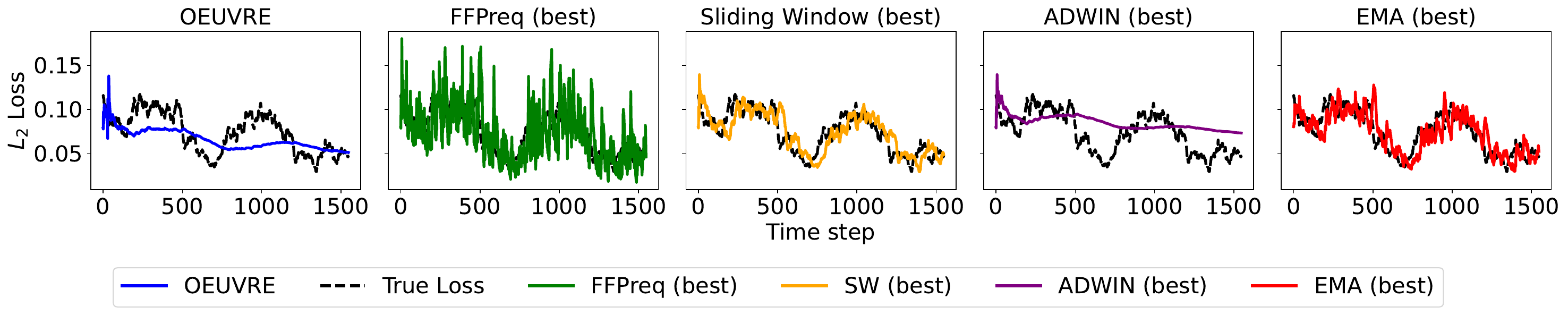}
  \caption{Performance of different algorithms on the Wine Quality dataset.}
  \label{fig:perf_comp_wine_quality_dataset}
\end{figure}

\begin{table}[h]
  \centering
  % Bikes results
  \begin{subtable}[t]{0.45\textwidth}
    \centering
    \begin{tabular}[t]{lrrr}
      \toprule
      Estimator & RMSE & MAE & Bias \\
      \midrule
      Prequential & 0.0293 & 0.0209 & -0.0012 \\
      ADWIN & 0.0245 & 0.0161 & 0.0012 \\
      SW (best) & 0.0142 & 0.0084 & -0.0003 \\
      EMA (best) & 0.0138 & 0.0080 & -0.0004 \\
      OEUVRE & 0.0158 & 0.0096 & -0.0010 \\
      FFPreq (best) & 0.0176 & 0.0098 & -0.0003 \\
      \bottomrule
    \end{tabular}
    \subcaption{Bikes dataset}
    \label{tab:perf_comp_bikes}
  \end{subtable}
  % Bike sharing dataset
  \begin{subtable}[t]{0.45\textwidth}
    \centering
    \begin{tabular}[t]{lrrr}
      \toprule
      Estimator & RMSE & MAE & Bias \\
      \midrule
      Prequential & 0.0054 & 0.0040 & -0.0019 \\
      ADWIN & 0.0053 & 0.0040 & 0.0004 \\
      SW (best) & 0.0033 & 0.0026 & -0.0002 \\
      EMA (best) & 0.0033 & 0.0026 & -0.0002 \\
      OEUVRE & 0.0035 & 0.0026 & -0.0004 \\
      FFPreq (best) & 0.0080 & 0.0055 & -0.0000 \\
      \bottomrule
    \end{tabular}
    \subcaption{Bike Sharing dataset}
    \label{tab:perf_comp_bike_sharing_dataset}
  \end{subtable}

  % Chick weights dataset
  \begin{subtable}[t]{0.45\textwidth}
    \centering
    \begin{tabular}{lrrr}
      \toprule
      Estimator & RMSE & MAE & Bias \\
      \midrule
      Prequential & 0.0520 & 0.0369 & -0.0369 \\
      ADWIN & 0.0520 & 0.0369 & -0.0369 \\
      SW (best) & 0.0155 & 0.0112 & -0.0111 \\
      EMA (best) & 0.0151 & 0.0108 & -0.0100 \\
      OEUVRE & 0.0171 & 0.0127 & -0.0127 \\
      FFPreq (best) & 0.0169 & 0.0110 & -0.0084 \\
      \bottomrule
    \end{tabular}
    \subcaption{Chick Weights dataset}
    \label{tab:perf_comp_chick_weights}
  \end{subtable}
  % Wine quality dataset
  \begin{subtable}[t]{0.45\textwidth}
    \centering
    \begin{tabular}{lrrr}
      \toprule
      Estimator & RMSE & MAE & Bias \\
      \midrule
      Prequential & 0.0230 & 0.0192 & 0.0110 \\
      ADWIN & 0.0230 & 0.0192 & 0.0110 \\
      SW (best) & 0.0151 & 0.0119 & 0.0004 \\
      EMA (best) & 0.0163 & 0.0126 & -0.0001 \\
      OEUVRE & 0.0193 & 0.0164 & -0.0044 \\
      FFPreq (best) & 0.0259 & 0.0200 & -0.0003 \\
      \bottomrule
    \end{tabular}
    \subcaption{Wine Quality dataset}
    \label{tab:perf_comp_wine_quality_dataset}
  \end{subtable}
  \caption{Performance comparison of different estimators on non-stationary datasets. OEUVRE achieves competitive RMSE, MAE, and bias when compared to the best baseline.}
\end{table}

\end{document}